\documentclass{article}

\usepackage[final]{neurips_2025}

\usepackage[utf8]{inputenc} 
\usepackage[T1]{fontenc}    
\usepackage{hyperref}       
\usepackage{url}            
\usepackage{booktabs}       
\usepackage{amsfonts}       
\usepackage{nicefrac}       
\usepackage{microtype}      
\usepackage{xcolor}         

\title{
Optimal Graph Clustering\\without Edge Density Signals
}

\author{%
  Maximilien Dreveton \\
  EPFL \\
  \texttt{maximilien.dreveton@epfl.ch} 
  \And
  Siyu (Elaine) Liu \\
  Stanford University \\
  \texttt{elaineliu@stanford.edu} 
  \AND
  Matthias Grossglauser \\
  EPFL \\
  \texttt{matthias.grossglauser@epfl.ch} 
  \And
  Patrick Thiran \\
  EPFL \\
  \texttt{patrick.thiran@epfl.ch}
}

\usepackage{macros_neurips}

\usepackage{bbm}
\usepackage{dsfont}
\usepackage{epsf}
\usepackage{fancyhdr}
\usepackage{graphics}
\usepackage{graphicx}
\usepackage{psfrag}

\usepackage{mathtools}
\usepackage{graphicx}
\usepackage{caption}
\usepackage{subcaption}
\newsavebox{\bigimage}
\usepackage{multirow}
\usepackage{float}
\usepackage[nodisplayskipstretch]{setspace} 
\usepackage[capitalise,noabbrev,nameinlink]{cleveref}
\usepackage{dsfont}
\usepackage{xspace}
\usepackage{pifont}%
\RequirePackage{forloop}
\RequirePackage{url}

\usepackage{amsthm}
\usepackage{amsfonts}
\usepackage{amsmath}
\usepackage{amssymb}

\newtheorem{theorem}{Theorem}
\newtheorem{lemma}[theorem]{Lemma}

\newtheorem{proposition}[theorem]{Proposition}
\newtheorem{assumption}{Assumption}

\newtheorem{example}{Example}

\newcommand{\tcZ}{\tilde{\cZ}}

\newcommand{\lambdain}{\lambda^{\rm in}}
\newcommand{\lambdaout}{\lambda^{\rm out}}

\newcommand{\CH}{\mathrm{CH}}

\newcommand{\Din}{\cD_{\rm in}}
\newcommand{\Dout}{\cD_{\rm out}}
\newcommand{\erf}{\operatorname{erf}}

\newcommand{\gammain}{\gamma_{\rm in}}
\newcommand{\gammaout}{\gamma_{\rm out}}

\usepackage{diagbox}

\usepackage{multirow}

\begin{document}

\maketitle

\begin{abstract}
 This paper establishes the theoretical limits of graph clustering under the Popularity-Adjusted Block Model (PABM), addressing limitations of existing models. In contrast to the Stochastic Block Model (SBM), which assumes uniform vertex degrees, and to the Degree-Corrected Block Model (DCBM), which applies uniform degree corrections across clusters, PABM introduces separate popularity parameters for intra- and inter-cluster connections. Our main contribution is the characterization of the optimal error rate for clustering under PABM, which provides novel insights on clustering hardness: we demonstrate that unlike SBM and DCBM, cluster recovery remains possible in PABM even when traditional edge-density signals vanish, provided intra- and inter-cluster popularity coefficients differ. This highlights a dimension of degree heterogeneity captured by PABM but overlooked by DCBM: \textit{local} differences in connectivity patterns can enhance cluster separability independently of \textit{global} edge densities. Finally, because PABM exhibits a richer structure, its expected adjacency matrix has rank between $k$ and~$k^2$, where $k$ is the number of clusters. As a result, spectral embeddings based on the top $k$ eigenvectors may fail to capture important structural information. Our numerical experiments on both synthetic and real datasets confirm that spectral clustering algorithms incorporating~$k^2$ eigenvectors outperform traditional spectral approaches. 
\end{abstract}


\section{Introduction}

Graph clustering is the task of partitioning the vertex set of a graph into non-overlapping groups such that vertices within the same group exhibit similar patterns or properties. As a fundamental task in the statistical analysis of networks, graph clustering plays a key role in revealing the underlying structure and functional organization of complex networks~\citep{avrachenkov2022statistical}. 

Most graph clustering algorithms are based on the assumption that vertices within the same cluster are more densely connected than vertices in different communities. In other words, intra-cluster edge-density is higher than inter-cluster edge-density. Under this premise, metrics such as modularity, graph cuts, or their variants are commonly used to motivate and design graph clustering algorithms. However, these methods fundamentally rely on the edge density as their primary input signal. This leads to a natural question: \textit{Is edge density essential for recovering clusters, or can other structural signals be exploited instead?} In this work, we demonstrate that the connection patterns of individual vertices can be exploited to recover clusters, even when intra-cluster and inter-cluster edge densities are equal.

\paragraph{Random graphs with cluster structure: block models with and without degree heterogeneity.}
Random graphs with cluster structure are often modeled using block models. Let $z \in [k]^n$ be a vector representing the cluster assignments of each vertex. For all the random graphs that we consider, the adjacency matrix $A \in \{0, 1\}^{n\times n}$ is assumed to be symmetric with zero diagonal and $ A_{ij} \weq A_{ji} \wsim \Ber( P_{ij} )$ for all $i  > j$, where $P_{ij} \in [0,1]$ is the probability of an edge between vertices~$i$ and $j$. The simplest block model supposes that 
\begin{align}
\label{eq:def_sbm_homogeneous}
 P_{ij} \weq 
 \begin{cases}
   p & \text{ if } z_i = z_j, \\
   q & \text{ otherwise.}
 \end{cases}
\end{align}
This model is often called the planted partition model, or the \new{stochastic block model} (SBM) with homogeneous interactions.\footnote{A block model is said to have homogeneous interactions if the entries $P_{ij}$ depends only on whether $z_i = z_j$ or $z_i \ne z_j$; otherwise, the model is said to have heterogeneous interactions. Our work focuses on models with heterogeneous interactions, with homogeneous interactions treated as a special case. However, for simplicity, in the Introduction we present results only for the homogeneous setting.} A known drawback of this model is that all vertices share the same expected degree. To mitigate this issue, \cite{karrer2011stochastic} proposed the \new{degree-corrected block model} (DCBM), where 
\begin{align}
\label{eq:def_dcbm_homogeneous}
 P_{ij} \weq 
 \begin{cases}
   \theta_i \theta_j p & \text{ if } z_i = z_j, \\
   \theta_i \theta_j q & \text{ otherwise.}
 \end{cases}
\end{align}
The quantities $\theta_1, \cdots, \theta_n$ are the degree-correction parameters. To ensure identifiability, these parameters are normalized such that $\sum_{i \colon z_i = a} \theta_i = n_a(z)$ for all $a \in [k]$, where $n_a(z) = \left| \left\{ i \colon z_i = a \right\} \right|$ denotes the size of cluster $a$.

However, the degree-correction parameter $\theta_i$ uniformly inflates or deflates the connection probabilities of vertex $i$ across all clusters. As a result, vertices with a large degree-correction parameter have more edges both within their own cluster and with other clusters. This makes it impossible to model vertices that exhibit higher connectivity exclusively within their own cluster. To mitigate this issue, \cite{sengupta2018block} introduced the \new{popularity adjusted block model} (PABM), where 
\begin{align*}
 P_{ij} \weq 
 \begin{cases}
     \lambdain_i \lambdain_j p & \text{ if } z_i = z_j, \\
     \lambdaout_i \lambdaout_j q & \text{ otherwise.}
 \end{cases}
\end{align*}
In this model, the quantity $\lambdain_{i}$ (resp., $\lambdaout_i$) is the popularity of vertex $i$ with other vertices within its own cluster (resp., with vertices in other clusters). These coefficients are normalized such that $\sum_{i \colon z_i = a} \lambdain_i = n_a(z)$ and $\sum_{i \colon z_i = a} \lambdaout_i = n_a(z)$ for all $a \in [k]$. This model allows for a vertex $i$ to be highly popular among its cluster (high $\lambdain_i$), but to be not necessarily popular ($\lambdaout_i = 1$) or even to be very unpopular (small $\lambdaout_i$) with vertices in other clusters. 

\paragraph{Optimal clustering error rate: from edge-density to popularity patterns}
An important question to assess the difficulty of the clustering task in a block model is the derivation of the \new{optimal error rate}. By optimal error rate, we refer to the minimum possible error that the best algorithm achieves when attempting to recover the true cluster assignment of all vertices. This error rate is typically measured in terms of the misclassification rate—that is, the proportion of vertices incorrectly assigned to their true clusters, up to a permutation of the labels. The optimal error rate reflects the information-theoretic limits of the clustering task, because it characterizes how well one could possibly do even with unlimited computational power, given the amount of signal and noise in the data. It also provides a benchmark to evaluate existing algorithms and guides the development of new methods that approach (either theoretically or empirically) these theoretical limits. 

Studying the effect of the different model parameters (such as sparsity or degree heterogeneity) on the error rate offers deep insight into the fundamental difficulty of the graph clustering problem across different network settings. Consider a SBM with $k$ clusters of same size $n/k$ and homogeneous interactions as in~\eqref{eq:def_sbm_homogeneous}. When $1/n \ll p,q \ll 1$, the optimal error rate is asymptotically~\citep{zhang2016minimax} 
\begin{align*}
 \exp\left( - \frac{n}{k} \left( \sqrt{p} - \sqrt{q} \right)^2 \right). 
\end{align*}
As $p$ and $q$ represent the intra-cluster and inter-cluster edge densities, respectively, the key quantity $( \sqrt{p} - \sqrt{q})^2$ in the expression above captures the influence of edge density: the larger the gap between~$p$ and $q$, the easier it is to recover the clusters. 

Next, consider a DCBM with $k$ clusters of same size $n/k$ and homogeneous interactions as in~\eqref{eq:def_dcbm_homogeneous}. Under some technical conditions on the degree-correction parameters, \cite{gao2018community} establishes that, when $p,q = o(1)$ with $p/q = O(1)$ and $p = \omega(1/n)$, the optimal error rate is asymptotically
\begin{align*}
  \frac1n\sum_i \exp\left( - \theta_i \frac{n}{k} \left( \sqrt{p} - \sqrt{q} \right)^2 \right). 
\end{align*}
Compared to the standard SBM, the difficulty of clustering now varies across vertices and is quantified by the term $\exp\left( - \theta_i n ( \sqrt{p} - \sqrt{q} )^2 / k \right)$, which depends on each vertex $i \in [n]$ and is monotonically decreasing in $\theta_i$. The optimal error rate corresponds to the average of these quantities over all vertices. This highlights the effect of degree heterogeneity: vertices with larger expected degree are easier to cluster, as their neighborhoods contain more information. 

However, the same key quantity $( \sqrt{p} - \sqrt{q} )^2$ representing the edge-density signal shows up in the DCBM error rate. Indeed, as mentioned earlier, the degree-correction parameters uniformly inflate or deflate the connection probabilities. As a result, the value of $\theta_i$ impacts the clustering difficulty of vertex $i$ in a predictable and monotonic way. This no longer holds in the PABM, which introduces a richer and more nuanced structure. The first major contribution of this work is to characterize the optimal error rate for clustering under the PABM. As the general expression is somewhat involved, we begin with the simplest case of $k=2$ clusters of equal size. In this setting, we establish that the optimal error rate is given by
\begin{align*}
 \frac1n\sum_{i \in [n] } \exp\left( - \frac{1}{2} \sum_{ j \in [n] } \left( \sqrt{ \lambdain_{i} \lambdain_j p} -  \sqrt{ \lambdaout_i \lambdaout_j q } \right)^2 \right). 
\end{align*}
As in the DCBM, the error rate in PABM is expressed as an average over the difficulty of clustering each individual vertex. However, in PABM, these per-vertex difficulties have a more intricate form, and we provide further insight in Sections~\ref{subsection:info_theoretic_divergence} and~\ref{subsection:comparison_other_block_models}. A particularly important observation is the following: suppose $p=q$, so that the expected numbers of intra-cluster and of inter-cluster edges are equal. In this case, the SBM and DCBM reduce to the \Erdos-\Renyi and Chung–Lu models, respectively, and cluster recovery is fundamentally impossible. Remarkably, this is not true for PABM: cluster recovery may still be possible provided the popularity coefficients $\lambdain_i$ and $\lambdaout_i$ are different. This reveals a novel aspect of degree heterogeneity captured by PABM but missed by DCBM: \textit{local} differences in intra- and inter-cluster popularity enhance the separability of clusters, even when traditional \textit{global} edge-density signals vanish. Another phenomenon, more subtle, occurs in PABM: the optimal error rate is \textit{not} monotonically increasing when the number of inter-cluster edges increases. We rigorously establish these phenomena in Examples~\ref{example:recovery_pabm_degree} and~\ref{example:recovery_pabm_xi}, and illustrate them in our numerical simulations.

\paragraph{Higher-order eigenvectors for clustering with popularity patterns}
Finally, we perform numerical experiments to evaluate the effectiveness of spectral clustering methods. When the adjacency matrix~$A$ is sampled from a block model, it can be decomposed as
$ A \weq P + X$, 
where $P$ is a low-rank matrix encoding the underlying structure, and $X$ is a random noise matrix with zero-mean sub-Gaussian entries. This decomposition forms the basis of spectral methods for graph clustering, where the general approach is to apply a clustering algorithm (such as $k$-means) to a low-dimensional embedding derived from a low-rank approximation of $A$.

In classical models like SBM and DCBM, when $p \ne q$, the rank of $P$ is equal to the number of clusters~$k$. However, in PABM, the situation is more complex: the rank of $P$ can be greater than $k$, but cannot be greater than $k^2$. This implies that embeddings based solely on the top-$k$ eigenvectors may miss important structural information. To address this, recent works propose spectral algorithms that incorporate~$k^2$ eigenvectors to better capture the richer structure of PABM~\citep{noroozi2021estimation,koo2023popularity}. Our numerical experiments demonstrate that these methods outperform traditional spectral approaches that rely only on~$k$ eigenvectors, both on synthetic and real datasets.

In the numerical section, we illustrate two surprising results discussed in the theoretical section: the non-monotonic behavior of the error with respect to edge density, and the ability to recover clusters even when $p=q$. While it would have been possible to use a greedy algorithm to approximate the MLE, we opted for spectral methods because of their widespread use and of their well-established effectiveness for clustering in block models. The experiments demonstrate that the phenomena highlighted in the theoretical section also arise when using spectral algorithms. They show that these behaviors are not merely mathematical artifacts stemming from the increased complexity of PABM relative to DCBM, but that they do occur in practice and are observable in real-world settings.

The paper is structured as follows. We derive the optimal error rate in PABM and provide some examples in Section~\ref{section:error_rates}. We present our numerical experiments in Section~\ref{section:numerical_experiments}. We discuss some related works in Section~\ref{section:related_work}. Finally, we conclude in Section~\ref{section:conclusion}. 

\paragraph{Notations} $\Ber(p)$, $\Exp(\lambda)$ and $\Uni(a,b)$ denote the Bernoulli distribution with parameter~$p$, the exponential distribution with parameter $\lambda$, and the uniform distribution over the interval $[a,b]$. We use the Landau notations $o$ and $O$, and write $f = \omega(g)$ when $g = o(f)$ and $f = \Omega(g)$ when $g = O(f)$. 

\section{Optimal Error Rate in Popularity-Adjusted Block Models}
\label{section:error_rates}

\subsection{Model Definition and Parameter Space}

We consider $n$ vertices partitioned into $k \ge 2$ disjoint blocks. The partition is encoded by a vertex-labeling vector $z^* = (z_1^*, \cdots, z_n^*) \in [k]^n$ so that $z_i^*$ indicates the cluster of vertex $i$. These~$n$ vertices interact pairwise, giving rise to undirected edges, and these pairwise interactions are grouped by a symmetric matrix $A \in \{0,1\}^{n \times n}$ called the adjacency matrix. The \new{Popularity Adjusted Block Model} supposes that, conditionally on the block structure, the upper-diagonal elements $(A_{ij})_{i > j}$ are independent Bernoulli random variables such that, conditionally on $z_i^*$ and $z_j^*$, 
\begin{align}
\label{eq:def_adjacency_matrix_pabm}
 A_{ij} \cond z_i^*, z_j^* \sim \Ber\left( \rho_n \lambda_{i z_j^*} \lambda_{j z_i^*} B_{z_i^* z_j^*} \right), 
\end{align}
 where $(\lambda_{i a} )_{i \in [n], a \in [k]}$ are the popularity parameters and $B \in \R_+^{k \times k}$ is the connectivity matrix across clusters. The parameter $\rho_n$ controls the graph sparsity, as the average degree is of order $n \rho_n$ when the following assumption is made. 
\begin{assumption}
 \label{assumption:scaling_parameters} The quantities $B_{ab}$ and $\lambda_{ia}$ are constant (so they do not scale with $n$) for all $i\in[n]$ and $a,b \in [k]$. 
\end{assumption}

Given a realization of a PABM, we aim to infer the latent block structure $z^*$. 
Let $\hz = \hz(A)$ be an estimate of $z^*$, and define the clustering error as
\begin{align}
\label{eq:def_haming_loss_function}
 \loss(z^*,\hz) \weq \frac1n \min_{\tau \in \Sym(k)} \ham(z^*, \tau \circ \hz),
\end{align}
where $\Sym(k)$ is the set of permutations of $[k]$ and $\ham(\cdot,\cdot)$ is the Hamming distance. We are interested in the expected loss of an estimator, namely $\E \left[ \loss(z^*,\hz(A) ) \right]$, where the expectation is taken with respect to the random variable~$A$ sampled from~\eqref{eq:def_adjacency_matrix_pabm}. 

\subsection{A Key Information-Theoretic Divergence}
\label{subsection:info_theoretic_divergence}

For any $z \in [k]^n$, denote $P_{ij}(z) = \rho_n \lambda_{i z_j} \lambda_{j z_i} B_{z_i z_j}$. 
To understand the difficulty of correctly clustering a given vertex $i$, we introduce an alternative cluster labeling $\tz^{ia} \in [k]^n$  such that $\tz^{ia}_j = z^*_j$ for all $j\ne i$, while $\tz^{ia}_i = a \in [k] \setminus \{z_i^*\}$. In other words, the cluster labeling $\tz^{ia}$ agrees with $z^*$ for all vertices except for~$i$, which is placed in cluster $a$ instead of being in cluster $z_i^*$. To shorten the notations, let $P^* = P(z^*)$ and $\tP^{ia} = P(\tz^{ia})$. 
The difficulty of correctly recovering the cluster of vertex $i$ depends on how hard it is to statistically distinguish whether the observed graph was generated from the true model $P^*$ or from the alternative model $\tP^{ia}$. This is a classical hypothesis testing problem: the more similar the distributions induced by $P^*$ and $\tP^{ia}$, the less distinguishable two graphs drawn from these two models are, and thus the harder it is to infer the correct cluster assignment for vertex  $i$. The statistical difficulty of this test is quantified by the Chernoff divergence $\Delta(i, a)$, which measures the exponential rate at which the error probability decays when testing between these two competing models. 
More precisely, 
\begin{align}
\label{eq:def_deltaia}
  \Delta(i, a) 
  & \weq \max_{t \in (0,1)} (1-t) \dren_t \left( \bigotimes_{j \ne i} \Ber \left( \tP^{ia}_{ij} \right), \bigotimes_{j \ne i} \Ber \left( P^*_{ij} \right) \right), 
\end{align}
where $\dren_t$ is the \Renyi divergence of order $t$.
Moreover, by using the linearity of \Renyi divergence with respect to multiplication and the sparsity of the model (that is, $P_{ij} = o(1)$ for all $i,j$), we have 
 \begin{align*}
  \Delta(i, a) 
  & \weq (1+o(1)) \max_{t \in (0,1)} \sum_{j \ne i} \left( t \tP_{ij}^{ia} + (1-t) P_{ij}^* - (\tP_{ij}^{ia})^t (P^*_{ij})^{1-t} \right). 
 \end{align*}
Among all alternative models $\tP^{ia}$, the most challenging to distinguish from the true model $P^*$ is the one with the smallest Chernoff divergence $\Delta(i,a)$. We thus define 
 \begin{align*}
  \Chernoff(i,z^*) \weq \min_{a \ne z_i^*} \, \Delta(i,a), 
 \end{align*}
which captures the hardest hypothesis testing problem associated with recovering the cluster of vertex~$i$. Intuitively speaking, the larger the value of $\Chernoff(i,z^*)$, the easier it is to correctly recover $z^*_i$, as all alternative models defined above are sufficiently different from $P^*$. 
The following assumption asserts that for every $i\in[n]$, the quantity $\Chernoff(i,z^*)$ is unbounded. This assumption is necessary to ensure that the recovery of $z_i^*$ is asymptotically possible. 
 \begin{assumption}
\label{assumption:scaling_chernoff}
 Suppose that $\min_{i \in [n]} \Chernoff(i,z^*) = \omega(1)$. 
\end{assumption}

\subsection{Main Result: Optimal Error Rate in PABM}

For any $z \in [k]^n$, denote by $n_a(z) = \sum_{i \in [n]} \1\{z_i=a\}$ the size of the cluster $a \in [k]$. Let $\pi \in [0,1]^k$ such that $\sum_a \pi_a = 1$ and define 
 \[
  \cZ_n(\pi,\epsilon) \weq \left\{ z \in [k]^n \colon \frac{n_a(z)}{n} \in [(1-\epsilon)\pi_a, (1+\epsilon)\pi_a] \ \forall a \in [k] \right\}.
 \]
 Let $\Lambda = (\lambda_{ia})_{i \in [n], a \in [k] }$ be a matrix with non-negative coefficients such that $\|\Lambda_{\cdot a}\|_1 = n \pi_a$ and $B \in \R_+^{k \times k}$ be a matrix of full rank.

\begin{theorem}[Lower-bound on the clustering error]
\label{thm:lower_bound}
Let $z^* \in \cZ_n(\pi,\epsilon)$ and $A$ being sampled from~\eqref{eq:def_adjacency_matrix_pabm}. Suppose Assumption~\ref{assumption:scaling_chernoff} holds. Then, there exists some $\eta = o(1)$ such that  
 \[
  \inf_{\hz}  \E \left[ \loss(z^*, \hz ) \right] \wge \frac{ (1-\epsilon) \min_{a} \pi_a } { 4 } \left( \frac1n \sum_{i \in [n] } e^{- \Chernoff(i,z^*)} \right)^{1+\eta}, 
 \]
 where the $\inf$ is taken over all estimators $\hz = \hz(A)$. 
\end{theorem}

\begin{theorem}[Achievability]
\label{thm:upper_bound}
Let $z^* \in \cZ_n(\pi,\epsilon)$ and $A$ being sampled from~\eqref{eq:def_adjacency_matrix_pabm}. Suppose Assumptions~\ref{assumption:scaling_parameters} and~\ref{assumption:scaling_chernoff} hold. Then, there exists an estimator $\hz$ such that 
 \[
   \E \left[ \loss(z^*,\hz ) \right] \wle \left( \frac1n \sum_{i \in [n] }  e^{-\Chernoff(i,z^*)} \right)^{1+\eta}
 \]
 for some $\eta = o(1)$
\end{theorem}

The gap between the lower bound (Theorem~\ref{thm:lower_bound}) and the achievability (Theorem~\ref{thm:upper_bound}) stems only from second-order terms. Indeed, the sequences $\eta$ appearing in Theorems~\ref{thm:lower_bound} and~\ref{thm:upper_bound} are not identical. Moreover, the multiplicative factor $\tfrac{ (1-\epsilon)\min_{a} \pi_a }{4}$ of constant order can be absorbed into the sequence~$\eta$, as the term $\tfrac{1}{n}\sum_{i \in [n]} e^{-\Chernoff(i,z^*)}$ vanishes as $n \to \infty$. We chose to display this factor explicitly in our bounds so that the sequence $\eta$ does not depend on the parameter $\epsilon$.

We show in Section~\ref{subsection:comparison_other_block_models} how Theorems~\ref{thm:lower_bound} and~\ref{thm:upper_bound} recover known results in SBM and DCBM, and in Section~\ref{subsection:optimal_rate_PABM}, how they reveal novel properties that did not exist in previous models, when they are specialized to PABM. Table~\ref{tab:block_model_variants} summarizes all three classes of block models considered in this paper.

\begin{table}[!ht]
 \centering
 {\footnotesize 
 \begin{tabular}{ c|c|c|c}
 \toprule 
 & SBM & DCBM & PABM \\ 
 \midrule 
 Homogeneous & 
 $P_{ij} = \begin{cases}
     p_0 \rho_n & ... \\
     q_0 \rho_n & ...
 \end{cases}$ & 
 $P_{ij} = \begin{cases}
     \theta_i \theta_j p_0 \rho_n & ... \\
     \theta_i \theta_j q_0 \rho_n & ... 
 \end{cases}$ &
 $P_{ij} = \begin{cases}
     \lambdain_{i} \lambdain_j p_0 \rho_n & ... \text{ if } z_i = z_j, \\
     \lambdaout_i \lambdaout_j q_0 \rho_n & ...\text{ otherwise.}
 \end{cases}$
 \\ \midrule
 Heterogeneous & $P_{ij} = B_{z_i z_j} \rho_n $ & $P_{ij} = \theta_i \theta_j B_{z_i z_j} \rho_n $ & $P_{ij} = \lambda_{i z_j} \lambda_{j z_i} B_{z_i z_j} \rho_n $ \\
 \bottomrule 
\end{tabular}
}
 \caption{Different expressions of the elements of the matrix $P$ for the block model variants considered in this paper. The quantity $\rho_n$ is related to the graph sparsity (as the expected degree is of order $n \rho_n$). All other quantities are strictly positive and independent of $n$.}
 \label{tab:block_model_variants}
\end{table}

\subsection{Recovering Known Optimal Error Rates in SBM and DCBM}
\label{subsection:comparison_other_block_models}

\paragraph{Inhomogeneous SBM}
Let $\lambda_{ia} = 1$ for all $i \in [n]$ and $a \in [k]$, so that we recover the SBM with inhomogeneous interactions in which $P_{ij} = \rho_n B_{z_i^* z_j^*}$. By linearity of the \Renyi divergence, we have 
\begin{align*}
 \Delta(i,a) & \weq \max_{t \in (0,1)} (1-t) \sum_{b = 1}^k n \pi_b \, \dren_t \left( \Ber \left( \rho_n B_{a b} \right), \Ber \left( \rho_n B_{z^*_i b} \right) \right) \\
 & \weq (1+o(1)) \, n \rho_n \underbrace{ \max_{t \in (0,1)} (1-t) \sum_{b = 1}^k \pi_b \left( t B_{ab} + (1-t) B_{z_i^* b} - B_{ab}^t B_{z_i^* b}^{1-t} \right) }_{ \CH_{AS}(a,z_i^*) }. 
\end{align*}
The quantity $\CH_{AS}(a,b)$ is called the Chernoff-Hellinger divergence~\citep{abbe2015community}. 
When $n \rho_n = \omega(1)$, we observe that
\begin{align*}
 \frac{1}{n} \sum_{i \in [n]} e^{- n \rho_n \min_{ a \in [k] \setminus \{ z_i^*\} } \CH_{AS}(a,z_i^*) } \weq e^{- (1+o(1)) n \rho_n \min_{a \ne b \in [k]} \CH_{AS}(a,b) },
\end{align*}
and we recover the instance-optimal error rate in SBM with inhomogeneous interactions~\citep{yun2016optimal}. 
Finally, the Chernoff-Hellinger divergence has a simple expression in the case of homogeneous interactions. Indeed, when $B_{ab} = p_0 \1(a=b) + q_0 \1(a\ne b)$ and the clusters are of equal-size ($\pi_a = 1/k$), the divergence simplifies to $\min\limits_{a \ne b \in [k]} \CH_{AS}(a,b) = \frac{ n \rho_n }{k}(\sqrt{p_0} - \sqrt{q_0})^2$. 

\paragraph{Degree-Corrected Block Model}
Suppose that $\lambda_{i z_j^*} = \theta_i$, so that the PABM boils down to a DCBM with homogeneous interactions in which $P_{ij} = \theta_i \theta_j B_{z_i^* z_j^*} \rho_n$. For the simplicity of the discussion, we consider cluster of equal-size (\textit{i.e.,} $\pi_a = 1/k$ for all $a \in [k]$), and homogeneous interactions (\textit{i.e.,} $B_{ab} = p_0 \1\{a = b \} + q_0 \1 \{a \ne b \}$). Consider a vertex $i$ in a cluster $z_i^*$ and let $a \in [k] \backslash \{ z_i^* \}$. 
We have $
  \Delta(i,a) \weq \theta_i \frac{ n \rho_n  }{ k } \left( \sqrt{p_0} - \sqrt{q_0} \right)^2,$
where we used $\sum_{i \colon z_i^* = a} \theta_i = 1$. Thus, we recover the asymptotic optimal error-rate $ \frac{1}{n} \sum_i e^{ - \theta_i \frac{ n \rho_n }{ k } \left( \sqrt{p_0} - \sqrt{q_0} \right)^2 }$ established in~\cite{gao2018community}.

\subsection{Optimal Error Rate in Homogeneous PABM}
\label{subsection:optimal_rate_PABM}
We now show how Theorems~\ref{thm:lower_bound} and~\ref{thm:upper_bound}, when applied to PABM, reveal novel phenomena. Suppose $B_{ab} = p_0 \1\{a = b \} + q_0 \1 \{a \ne b \}$ and $ \lambda_{i a} = \lambdain_i \1\{ z_i^* = a \} + \lambdaout_i \1\{ z_i^* \ne a \}$. We have 
\begin{align*}
 \Chernoff(i, z^*) = \frac{n \rho_n}{2 k} \left( \delta_{z_i^*} + \min_{a \ne z_i^*} \delta_a \right) 
 \text{ where }  
 \delta_{b} = \frac{1}{n/k} \sum_{j \colon z_j^* = b} \left( \sqrt{\lambdain_i \lambdain_j p_0} - \sqrt{ \lambdaout_i \lambdaout_j q_0 } \right)^2.
\end{align*}

\begin{proposition}
\label{prop:recovery_rate_pabm}
 Consider a PABM with homogeneous interactions with $p_0 q_0 > 0$, and $k$~equal-size communities. 
 Suppose that $\lambdaout_1 = \cdots = \lambdaout_n = 1$ and that the coefficients $\lambdain_1, \cdots, \lambdain_n$ are sampled iid from $\Uni(1-c,1+c)$ with $c \in (0,1)$. Denote $\gamma_c  = \frac{1}{3c} \left( (1+c)^{3/2} - (1-c)^{3/2} \right)$. We have 
 $$
 \frac1n \sum_{i \in [n] }  e^{-\Chernoff(i,z^*)} \weq (1+o(1)) e^{ - \frac{n \rho_n}{k} q_0 (1-\gamma_c^2) } J_n,
 $$
 where  
 \begin{align*}
    J_n \weq  \frac{k }{ 2 c p_0 n \rho_n } \left( e^{-\frac{n \rho_n}{k} p_0 u_+^2} -e^{ - \frac{n \rho_n}{k} p_0 u_-^2} \right) + \frac{\gamma_c}{2c} \sqrt{ \frac{k \pi}{ n \rho_n q_0 } } \left (\erf \left( \frac{n \rho_n}{k} p_0 u_+ \right) - \erf\left( \frac{n \rho_n}{k} p_0 u_- \right) \right),
 \end{align*}
with $u_{\pm} = \sqrt{1\pm c} - \gamma_c \sqrt{q_0 / p_0}$ and $\erf(t) = 2 / \sqrt{\pi} \int_0^t e^{-t^2} \diff t$ is the Gauss error function. 
\end{proposition}

Although the expression of $J_n$ is quite involved, we can give two interesting particular cases. Firstly, to highlight the effect of degree heterogeneity, suppose that $p_0 = q_0$. In this extreme case where we expect the same number of interactions within and across clusters, the only information comes from the degree heterogeneity. Therefore, many existing graph clustering algorithms are expected to fail (as indeed shown in Section~\ref{section:numerical_experiments}). However, the quantity $\Chernoff(i, z^*)$ does not vanish. Therefore, even in this extreme setting, consistent recovery is possible, as highlighted in the following example. This stands out in stark contrast to the standard and degree-corrected block models, where setting $p_0 = q_0$ causes the model to collapse into an \Erdos-\Renyi graph and a Chung-Lu graph, respectively—both of which contain no information about the underlying cluster structure. 

\begin{example}
\label{example:recovery_pabm_degree}
 Consider the setting of Proposition~\ref{prop:recovery_rate_pabm} where $p_0 = q_0$. 
 For $c=0$ the model reduces to an \Erdos-\Renyi graph with edge-probability $p$, and thus no recovery is possible. 
 However, for $c>0$, 
 \begin{align*}
    \frac1n \sum_{i \in [n] } e^{-\Chernoff(i,z^*)} \weq (1+o(1)) e^{ - \frac{n \rho_n}{k} p_0 (1-\gamma_c^2)}. 
 \end{align*}
 Observe that $1 - \gamma_c^2$ is  strictly positive and increasing in $c \in (0,1]$. Thus, if $c > 0$, the optimal error rate satisfies $
    \frac1n \sum_{i \in [n] } e^{-\Chernoff(i,z^*)} \weq o(1)
$, yielding that cluster recovery is possible. Moreover, this rate is monotonically decreasing in $c$. 
\end{example}

In the following example, we fix $c \in (0,1)$ and $p_0 > 0$, and we let $q_0$ vary between $0$ and $p_0$.  

\begin{example}
\label{example:recovery_pabm_xi}
 Consider the setting of Proposition~\ref{prop:recovery_rate_pabm} with $\xi = q_0 / p_0 \in (0,1]$. The quantity $\exp({ - \frac{n \rho_n}{k} p_0 \xi (1-\gamma_c^2) }) \times J_n$ is \textit{not} monotonically increasing in $\xi$, but instead first \textit{increases} with $\xi$, reaches some maximum value, and then \textit{decreases}. We illustrate this in Figure~\ref{fig:optimal_error_rate_uniform_ones} in Appendix~\ref{appendix:discussion_examples}. 
\end{example}
The intuition behind Example~\ref{example:recovery_pabm_xi} is as follows. As $\xi=0$ the graph is disconnected and the $k$ largest components are aligned with the $k$ clusters. Hence, the difficulty of clustering is at its lowest and can only increases with $\xi$, as additional edges are inter-cluster edges and act as noise. However, when~$\xi$ becomes large enough, the difference between the intra- and inter-connectivity patterns, governed by the $\lambdain$ and $\lambdaout$, becomes more pronounced. As a result, this provides additional information that can be exploited for clustering (as in Example~\ref{example:recovery_pabm_degree}). This leads to a trade-off between the benefit brought by the absence of any inter-cluster edges (for learning from well-separated clusters) and the benefit brought by their presence in large numbers (for learning from popularity patterns). 
This non-monotonic behavior is specific to PABM and does not occur in~DCBM. 
Finally, this non-monotonicity is \textit{not} an artifact of setting $\lambdaout_i = 1$ and sampling the $\lambdain_i$ from a uniform distribution. This choice was made because of the difficulty to derive a closed-form expression for the optimal error rate. In Appendix~\ref{appendix:discussion_examples}, we show numerically that these phenomena persist under alternative distributions for the coefficients $\lambdain$ and~$\lambdaout$.

\section{Numerical Experiments}
\label{section:numerical_experiments}

In this section, we numerically evaluate the performance of several existing variants of spectral clustering on both synthetic and real-world datasets.\footnote{Our code is available at \url{https://github.com/mdreveton/neurips-pabm}.} Specifically, we compare the following variants:
\begin{itemize} 
 \item \textit{sbm}: Lloyd's algorithm applied to the embedding formed by the $k$ largest (in magnitude) eigenvectors of the adjacency matrix (see Algorithm~\ref{algo:sc_sbm}); 
 \item \textit{dcbm}: Lloyd's algorithm applied to an estimate $\hP$ of the connectivity matrix $P$, constructed using the $k$ largest (in magnitude) eigenvectors of the adjacency matrix (see Algorithm~\ref{algo:sc_dcbm}); 
 \item \textit{pabm}: subspace clustering applied to the embedding formed by the $k^2$ largest (in magnitude) eigenvectors of the adjacency matrix (see Algorithm~\ref{algo:sc_pabm}); 
 \item \textit{osc}: the spectral clustering variant described in Algorithm~\ref{algo:osc}; 
 \item \textit{sklearn}: Lloyd's algorithm applied to the embedding formed by the $k$ smallest eigenvectors of the graph's normalized Laplacian, corresponding to the implementation available in the \textit{scikit-learn} library (see Algorithm~\ref{algo:sc_scikit-learn}). 
\end{itemize}

The \textit{sbm} and \textit{dcbm} variants are tailored for graphs generated from SBM and DCBM, respectively, and are known to recover clusters accurately under these models~\citep{zhang2024fundamental,gao2018community}. In contrast, PABM exhibits a more complex structure, as the rank of the matrix $P$ can exceed $k$, but cannot be greater than $k^2$. We refer to~\citep{koo2023popularity,noroozi2021estimation} and to Section~\ref{subsec:rank_pabm} of the Appendix for examples. To accommodate this higher-rank structure, the \textit{pabm} and \textit{osc} variants rely on an embedding based on $k^2$ eigenvectors rather than the traditional $k$, allowing them to capture the higher-rank structure of PABM more effectively.

In the finalization phase of the manuscript, we became aware of two more algorithms designed for community recovery in PABM, namely Thresholded Cosine Spectral Clustering (\textit{tcsc}) and a Greedy Subspace Projection Clustering (\textit{gspc}), introduced in~\cite{yuan2025strongly} and in~\cite{bhadra2025unified}, respectively. To avoid overburdening this section, we refer the interested reader to the Appendix~\ref{appendix:additional_clustering_algos} for a description of these algorithms. We also provide in the Appendix~\ref{appendix:description_algos} the pseudo-code of all the algorithms.

In all experiments, we report the accuracy, defined as one minus the loss in~\eqref{eq:def_haming_loss_function}. It is equal to the proportion of correctly clustered vertices. 


\subsection{Synthetic Data Sets}
\label{subsec:expe_synthetic}

We first consider homogeneous PABM whose interaction probabilities are given by 
\begin{align}
\label{eq:experiments_homogeneous_PABM}
 P_{ij} \weq 
 \begin{cases}
  \lambdain_i \lambdain_j \rho & \text{ if } z_i^* = z_j^*, \\ 
  \lambdaout_i \lambdaout_j \xi \rho & \text{ otherwise.}
 \end{cases}
\end{align}
 The parameter $\rho \in (0,1)$ controls the overall sparsity of the network, while the parameter $\xi \in [0,1]$ controls the fraction of edges across clusters (in particular, $\xi=0$ implies no inter-cluster edges while $\xi=1$ implies the same expected number of edges between any pair of clusters). As in Examples~\ref{example:recovery_pabm_degree} and~\ref{example:recovery_pabm_xi}, we let $\lambdaout_i = 1$ and sample the coefficients $\lambdain$ from the uniform distribution in $(1-c,1+c)$. 
 
 In Figure~\ref{fig:homogeneous_unif_one_varying_c}, we let $\xi = 1$ and vary $c$. This is precisely the setting of Example~\ref{example:recovery_pabm_degree}. We observe that \textit{pabm} and \textit{osc} variants, which are specifically designed for PABM, recover the clusters when $c$ is large enough, whereas the variants tailored for SBM and DCBM fail to do so. This illustrates that \textit{pabm} and \textit{osc} successfully learn the clusters without edge-density signal by using the difference in the individual vertex degree connectivity patterns. 
 In Figure~\ref{fig:homogeneous_unif_one_varying_xi}, we set $c = 0.8$ and let $\xi$ vary. We observe that the acuracy of \textit{pabm} and \textit{osc} is \textit{not} monotonically decreasing with $\xi$. In fact, it goes to a minimum value before increasing again. This illustrates the phenomenon described in Example~\ref{example:recovery_pabm_xi}. In contrast, the accuracy obtained by \textit{sbm} and \textit{dcbm} variants monotonically decreases, because increasing $\xi$ from $0$ to $1$ monotonically decreases the edge-density signal.\footnote{In both cases, the accuracy achieved by \textit{sklearn} matches that of \textit{dcbm} and is omitted from the figures.}

 \begin{figure}[!ht]
 \centering
 \begin{subfigure}[b]{0.4\textwidth}
  \includegraphics[width=1.0\linewidth]{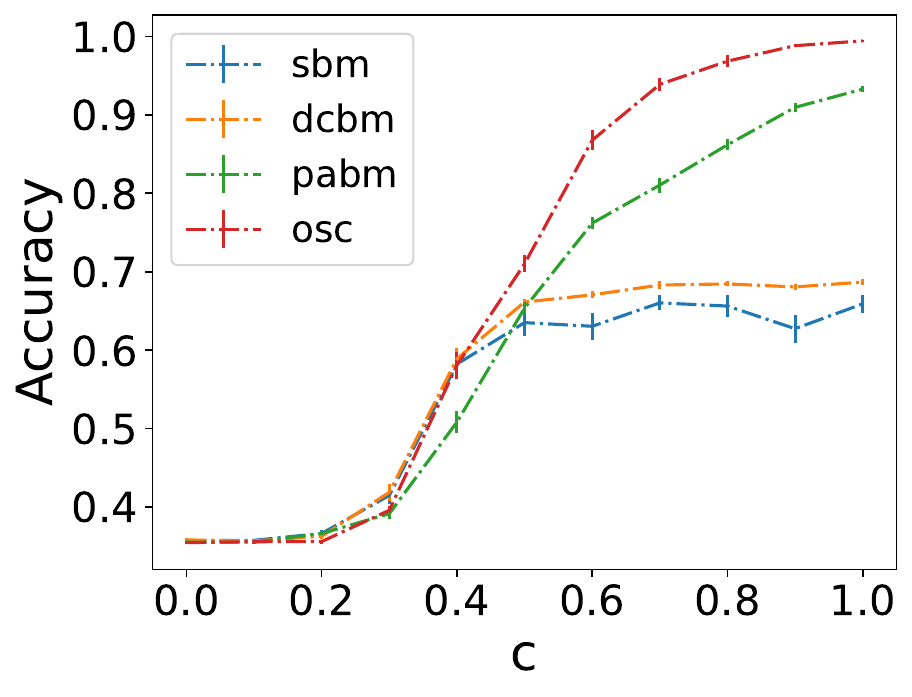}
  \caption{$\xi = 1$}
  \label{fig:homogeneous_unif_one_varying_c}
 \end{subfigure}
 \hfill 
 \begin{subfigure}[b]{0.4\textwidth}
  \includegraphics[width=1.0\linewidth]{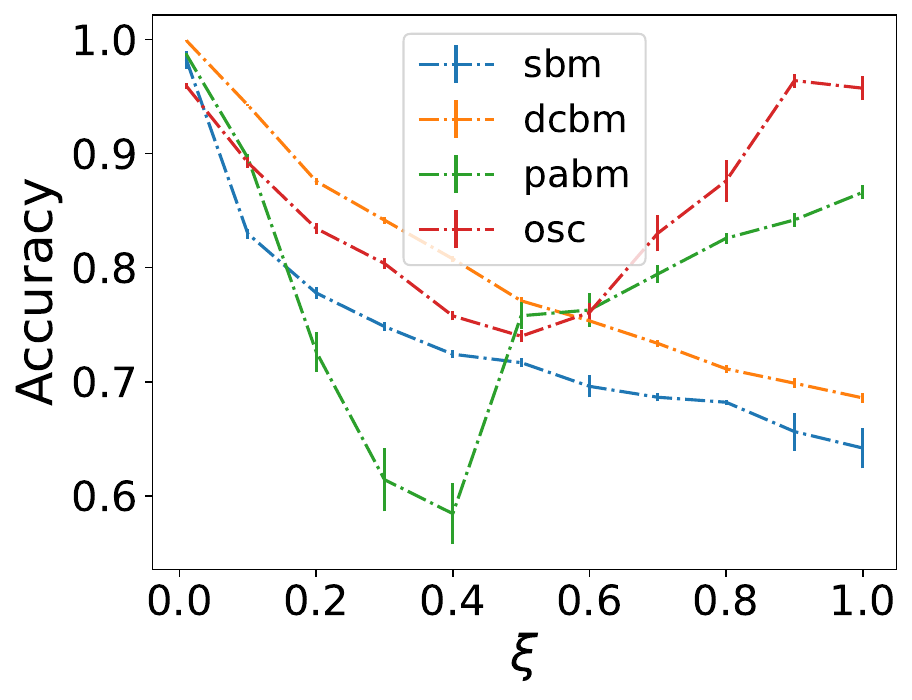}
  \caption{$c = 0.8$}
  \label{fig:homogeneous_unif_one_varying_xi}
 \end{subfigure}
 \caption{Performance of graph clustering on homogeneous PABM, where the matrix~$P$ is given in Equation~\eqref{eq:experiments_homogeneous_PABM}. We sampled graphs with $n=900$ vertices in $k=3$ clusters of same size, average edge density $\rho = 0.05$. In both figures, the $\lambdain_i$ are iid sampled from $\Uni(1-c,1+c)$ and $\lambdaout_i = 1$ for all $i$. Accuracy is averaged over 15 realizations, and error bars show the standard errors.}
 \label{fig:homogeneous_uniform}
\end{figure}

To further highlight the impact of the embedding dimension on the clustering accuracy, we plot in Figure~\ref{fig:homogeneous_uniform_embeddingDimension} the accuracy of the different spectral clustering variant as a function of embedding dimension~$d$. We observe that the performance of pabm and osc improves significantly as the dimension increases from $d=3$ to $d=6$, after which it reaches a plateau. In contrast, the performance of the sbm and dcbm variants remains unchanged with increasing $d$. 

 \begin{figure}[!ht]
 \centering
 \begin{subfigure}[b]{0.4\textwidth}
  \includegraphics[width=1.0\linewidth]{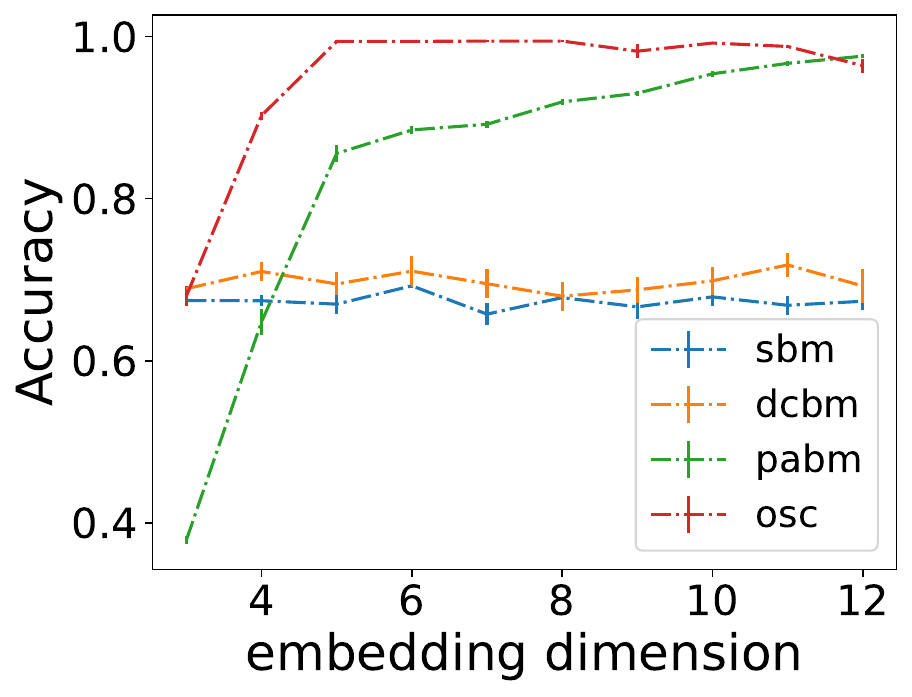}
  \caption{$\lambdaout_i = 1$}
 \end{subfigure}
 \hfill 
 \begin{subfigure}[b]{0.4\textwidth}
  \includegraphics[width=1.0\linewidth]{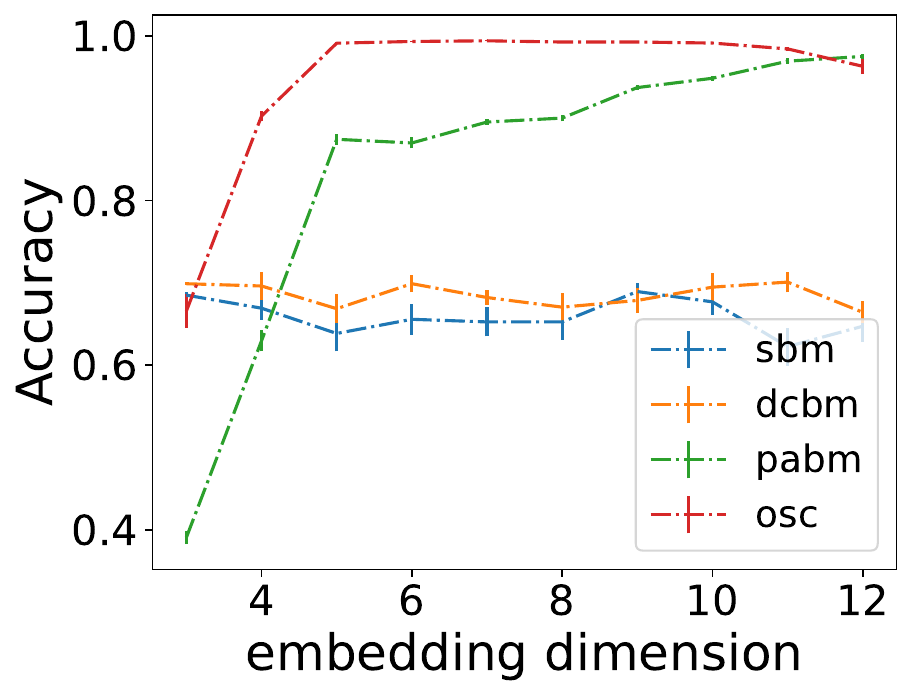}
  \caption{$\lambdaout_i \sim \Uni(0,2)$}
 \end{subfigure}
 \caption{Effect of the embedding dimension on the performance of graph clustering on homogeneous PABM, where the matrix~$P$ is given in Equation~\eqref{eq:experiments_homogeneous_PABM}. We sampled graphs with $n=900$ vertices in $k=3$ clusters of same size, average edge density $\rho = 0.05$. In both figures, the $\lambdain_i$ are iid sampled from $\Uni(0,2)$. Accuracy is averaged over 15 realizations, and error bars show the standard errors.}
 \label{fig:homogeneous_uniform_embeddingDimension}
\end{figure}

\subsection{Real Data Sets}

In this section, we show on real datasets that spectral algorithms that use more eigenvectors such as \textit{pabm} and \textit{osc} outperform the traditional variants that use only $k$ eigenvectors. Table~\ref{tab:statistics_real_datasets} in Appendix~\ref{appendix:real_datasets_description} summarizes some statistics of the dataset used. Table~\ref{tab:accuracy_real_datasets} shows the accuracy obtained by the different variants of spectral clustering on the real data sets. 

\begin{table}[!ht]
 \centering
\begin{tabular}{l | rrrrrrr}
\toprule
 &  sbm &  dcbm &  pabm &  osc &  tcsc &  gspc &  sklearn \\
\midrule
politicalBlogs & 0.63 & \textbf{0.95} &  0.91 & \textbf{0.95} & 0.65 & \textbf{0.95} & 0.52 \\
liveJournal-top2 & 0.56 &  0.61 & \textbf{0.99} & 0.59 &  \underline{0.98} & 0.60 & \textbf{0.99} \\
citeseer & 0.27 & 0.38 & 0.45 & \underline{0.56} & 0.33 & 0.51 & \textbf{0.58} \\
cora & 0.34 & 0.37 & \textbf{0.47} & \textbf{0.47} & 0.30 & \underline{0.42} & 0.27 \\
mnist & 0.44 &  0.54 &  \textbf{0.88} & 0.74 & 0.11 & 0.79 & 0.78 \\
fashionmnist & 0.22 &  0.41 &  \textbf{0.63} & 0.61 & \underline{0.60} & 0.50 & \underline{0.60} \\
cifar10 & 0.17 & 0.43 & \textbf{0.74} & 0.58 & 0.49 & 0.62 & \underline{0.71} \\
\bottomrule
\end{tabular}
 \caption{Accuracy of several spectral clustering variants on real data sets.}
 \label{tab:accuracy_real_datasets}
\end{table}

\section{Related Work}
\label{section:related_work}

\paragraph{Optimal clustering error rate.}
A rich line of research focused on characterizing the optimal error rates for clustering in stochastic block models and their variants. Early results established the minimax error rate in the SBM~\citep{zhang2016minimax}, while later work extended these insights to more general models such as the degree-corrected block model~\citep{gao2018community}. Further developments have addressed more complex network structures, such as categorical edge types~\citep{yun2016optimal}, weighted interactions~\citep{Xu_Jog_Loh_2020}, and more general interaction patterns~\citep{avrachenkov2020community}. These studies leverage information-theoretic tools to derive minimax bounds and to uncover the fundamental limits of clustering error. 
Parallel developments have taken place in the mixture model literature, where optimal error rates have been studied extensively, particularly in Gaussian mixture models~\citep{lu2016statistical,cai2019chime,chen2024achieving} and in more general mixture models~\citep{dreveton2024universal}. In both settings, a central objective is to understand how the separation between components governs the intrinsic difficulty of the clustering task.

\paragraph{Clustering with higher-order eigenvectors.}
Several studies have identified the benefits of incorporating higher-order eigenvectors beyond the first $k$ in spectral graph clustering. In networks whose connections depend on both cluster membership and spatial position, \cite{avrachenkov2021higher} demonstrated that the second eigenvector of the graph Laplacian typically aligns with the geometric structure rather than with the cluster structure. As a result, traditional spectral methods that rely solely on the leading eigenvectors often produce geometric partitioning that fails to accurately capture the underlying cluster structure. Their analysis reveals that incorporating additional eigenvectors beyond the conventional first $k$ can provide crucial information for distinguishing between geometric proximity and actual cluster membership. 

In sparse networks with strong degree heterogeneity—where some vertices have significantly higher degree than others—spectral clustering based on the top $k$ eigenvectors of the adjacency matrix often fails. In such cases, the leading eigenvectors tend to localize around high-degree vertices, rather than capturing the underlying cluster structure. Trimming-based approaches have been proposed to mitigate this issue by down-weighting or removing influential high-degree vertices~\citep{le2017concentration}. Alternatively, using the normalized Laplacian shifts the problem: its leading eigenvectors may become concentrated on peripheral substructures, such as dangling trees, while the cluster signal may still lie in higher-order eigenvectors. To address this, regularization techniques have been introduced to stabilize the spectral embedding and improve clustering performance~\citep{qin2013regularized}. 

Although the previous paragraphs illustrate two different settings where higher-order eigenvectors are crucial for uncovering cluster structure, they also share a key limitation: the leading eigenvectors are largely uninformative, and only the higher-order ones carry meaningful clustering information. PABM is fundamentally different, as potentially all $k^2$ eigenvectors can be informative for clustering. This richer spectral structure opens new avenues for designing more effective spectral algorithms.

\section{Conclusion}
\label{section:conclusion}

We established the optimal error rate for clustering under the PABM, providing a precise information-theoretic characterization of the fundamental limits of clustering in this rich and flexible model. Our results highlight how heterogeneity in vertex popularity fundamentally alters the clustering landscape, and how this is reflected in the spectral structure of the network. While our analysis provides a solid theoretical foundation, several important questions remain open. A deeper theoretical understanding of practical algorithms such as OSC and subspace clustering remains a key challenge. 
Another important direction for future work is model selection: developing principled methods to distinguish between models such as DCBM and PABM, and to infer key parameters like the number~$k$ of clusters or the rank of the connection probability matrix $P$. Addressing these challenges is essential to translate theoretical insights into robust, data-driven tools for network analysis. 


\bibliographystyle{chicago}
\bibliography{biblio.bib}

\appendix
\section{Additional Discussions on the Theoretic Results}

\subsection{Instance-Optimal versus Minimax Setting}

 In our study of optimal clustering rates in PABM, we did not made any assumption on the matrix~$B$ (beyond being symmetric). 
 As a result, our analysis derives the optimal error rate for a specific instance of PABM (similar to the instance-wise analysis in~\cite{yun2016optimal} for the edge-labeled SBM) rather than a minimax error rate (as in~\cite{zhang2016minimax} for SBM and \cite{gao2018community} for DCBM). Both approaches are valuable: the minimax framework requires defining a parameter space to which $B$ belongs (typically the space of matrices having diagonal values larger than or equal to $p$ and off-diagonal values smaller than or equal to $q$, but offers no guarantees when the matrix $B$ lies outside this space, while the instance optimal-rate restricts to a specific but arbitrary matrix~$B$. 

 Moreover, we wish to emphasize an important point: a rate-optimal algorithm in the minimax setting may not be rate-optimal for specific instances, even when those instances fall within the defined parameter space. For example, Lloyd’s algorithm is minimax-optimal over the class of sub-Gaussian mixture models \cite{lu2016statistical}, but it fails to be instance-optimal for Gaussian mixture models with anisotropic covariance structures~\cite{chen2024achieving}. 

 For the parameter space described two paragraphs above (matrix $B$ with diagonal values larger than or equal to $p$ and off-diagonal values smaller than or equal to $q$), the worst-case rate for SBM and DCBM arises when $B_{aa} = p$ for all $a \in [k]$ and $B_{ab} = q$ for all $a \ne b$, leading to a minimax rate involving the term $(\sqrt{p} - \sqrt{q})^2$. In contrast, for PABM, the situation is more complex because of the additional dependence on the individual parameters $\lambda_{ia}$. As a result, we do not believe that a simple closed-form expression for the minimax rate in PABM is attainable. Indeed, as shown in Example~\ref{example:recovery_pabm_xi}, reducing the gap between $p$ and~$q$ does not necessarily increases the optimal error-rate.

\subsection{Overview of the Proofs}

The overall structure of the proofs for Theorems~1 and~2, which establish the optimal error rate, is similar to that used for SBM and DCBM (in~\cite{zhang2016minimax} and~\cite{gao2018community}, respectively). However, the PABM setting introduces additional technical complexity that requires a more refined analysis.

 (i) For the lower-bound, a first challenge is to address the minimum over all permutations in the definition of the error loss. Hence, rather than directly examining $\inf_{ \hat{z} \in [k]^n } \mathbb{E}[ n^{-1} \rm{loss}(z^*, \hat{z}) ]$, we follow previous works such as~\cite{zhang2016minimax,gao2018community} and focus on a sub-problem $\inf_{ \hat{z} \in \mathcal{Z} } \mathbb{E}[ \mathrm{loss}(z^*, \hat{z}) ]$, where $\mathcal{Z} \subset [k]^n$ is chosen such that $\mathrm{loss}(z^*, \hat{z}) =  \ham( z^*, \hat{z} ) / n$ for all  $z^*, \hat{z} \in \mathcal{Z}$. This sub-problem is simple enough to analyze, while still capturing the hardness of the original clustering problem. Next, we use a result from \cite{dreveton2024universal} to show that the Bayes risk $\inf_{\hat{z}_i} \pr ( \hat{z}_i \ne z_i^*)$ for the misclustering of a single vertex $i$ is asymptotically $e^{-(1+o(1)) \mathrm{ Chernoff}(i,z^*)}$. 

 More precisely, \cite[Lemma~2]{dreveton2024universal} establishes the worst-case error rate for a binary hypothesis testing problem where the observed random variable is drawn from either distribution $f_1$ or $f_2$ (corresponding to hypothesis $H_1$ and $H_2$, respectively). Both $f_1$ and $f_2$ are arbitrary and known probability density functions. By the Neyman–Pearson lemma, the likelihood ratio test (equivalent to the MLE in this context) minimizes the probability of error, thereby ruling out all other estimators for this problem. The error of the MLE is then upper-bounded using Chernoff’s method and lower-bounded using a large deviation argument. In our setting, the hypothesis is formulated in Equation~\eqref{eq:def_hypothesis_problem}, where $f_1$ and $f_2$ are product distributions of Bernoulli random variables.

 (ii) The proof of the achievability is however more involved, and required a new approach. It begins, similarly to prior work on block models, by upper-bounding $\E[ \mathrm{loss}(z^*, \hat{z}) ]$ (where $\hat{z}$ is the MLE) by $\sum_m \pr( \mathrm{Ham}(z^*, \hat{z}) = m)$. Thus, the core difficulty relies in upper-bound the quantities $\pr(L(z) > L(z^*))$ for any $z$ such that $\mathrm{Ham}(z^*, z) = m$ (where $L(z)$ denotes the likelihood of $z$ given an observation of $A$). This is more challenging in PABM than in SBM and DCBM. Indeed, unlike in SBM (and to some extent DCBM), the likelihood ratio $L(z)/L(z^*)$ cannot be easily simplified. As for SBM and DCBM, we rely on Chernoff bounds to obtain $\pr(L(z) > L(z^*)) \le \mathbb{E} [ e^{t \log (L(z) /  L(z^*) } ]$ for any $t>0$. But, in SBM and DCBM, one can use $t = 1/2$ and obtain clean exponential bounds whose terms are $\exp(- \theta_u \theta_v (\sqrt{p}-\sqrt{q})^2)$. For PABM, the optimal $t$ to use depends intricately on the misclassified set $\{ u \colon z_u \ne z_u^*\}$, and thus on $z$ itself. To address this, we adopt a more refined approach: we decompose the upper bound into three components $T_1(t)$, $T_2(t)$, and $T_3(t)$, and select a tailored value of $t$ for each labeling $z$. This additional complexity distinguishes our analysis from earlier work and reflects the greater structural richness of PABM compared to SBM and DCBM.

\subsection{Extension when Assumption~\ref{assumption:scaling_chernoff} Fails}

The situation when Assumption~\ref{assumption:scaling_chernoff} fails is slightly delicate. Suppose firstly that Assumption~\ref{assumption:scaling_chernoff} fails such that $\max_i \mathrm{Chernoff}(i,z^*) = O(1)$. In that case, the optimal clustering error cannot vanish. Indeed, using arguments similar to those in~\cite{zhang2016minimax}, we can establish that the optimal error rate is lower-bounded by a non-zero constant $c > 0$. More generally, we introduce the set $S = \{ i \in [n] \colon \mathrm{Chernoff}(i,z^*) = O(1) \}$ of vertices having a non-vanishing error of being misclustered. Assumption~\ref{assumption:scaling_chernoff} fails whenever $S \ne \emptyset$. By refining the proof of Theorem~\ref{thm:lower_bound}, we can obtain a lower-bound for the clustering error of any algorithm of the form: 
 \begin{align*}
   \inf_{\hz}  \E \left[ \loss(z^*, \hz ) \right] \wge \frac{ (1-\epsilon) \min_{a} \pi_a } { 4 } \left(   \frac{1}{|S^c|} \sum_{i \in S^c} e^{ - \mathrm{Chernoff}(i,z^*) } + c \frac{|S|}{n} \right)^{1+\eta}. 
 \end{align*}
 This decomposition reflects that a constant fraction of nodes (those in $S$) are intrinsically hard to classify, while the rest exhibit standard exponential error decay. When Assumption~\ref{assumption:scaling_chernoff} holds, $S = \emptyset$ and the lower bound matches the result in Theorem~\ref{thm:lower_bound}. (And observe that the case $\max_i \mathrm{Chernoff}(i,z^*) = O(1)$ discussed earlier is equivalent to $|S|=n$, and we recover $\inf_{\hz}  \E \left[ \loss(z^*, \hz ) \right] \ge c$ for some non-vanishing constant $c>0$.)

 Showing that the MLE attains this bound when $S \ne \emptyset$ appears plausible but requires additional technical work.

\section{Proof of Theorem~\ref{thm:lower_bound}}

\subsection{Clustering one Vertex at a Time: the Genie-aided Problem}

Let $i \in [n]$ and suppose a genie gives you $z^*_{-i}$, \textit{i.e.,} the community labels of all nodes but $i$. Denote $H_a^{(i)} \colon z_i^* = a$ the hypothesis that node $i$ belongs to the cluster $a \in [k]$. Letting $ X = A_{i\cdot}$ being the $i$-th row of the adjacency matrix, the hypothesis testing resumes to 
  \begin{align}
  \label{eq:def_hypothesis_problem}
    H_a^{(i)} \ \colon \ X \wsim \bigotimes_{j \ne i} \Ber \left( \rho_n \lambda_{i z^*_j} \lambda_{j a} B_{z^*_j a} \right).
  \end{align}
  The worst-case error of a testing procedure $\phi \colon \{0,1\}^{n-1} \to \{1, \cdots, k\}$ is  
  \[
   r(\phi) \weq \max_{a \ne b} \pr \left( \phi(X) = a \cond H_b \right). 
  \]
  By the Neyman-Pearson lemma, we have $\phi^{\MLE} = \argmin_{\phi} r(\phi)$ where
  \begin{align*}
    \phi^{\MLE}( A_{i\cdot} ) \weq \argmax_{ a \in [k] } \prod_{j \ne i} \left( 1 -\rho_n \lambda_{iz_j^*} \lambda_{ja} B_{z^*_j a} \right)^{1-A_{ij}}  \left( \rho_n \lambda_{iz_j^*} \lambda_{ja} B_{z^*_j a} \right)^{A_{ij}}. 
  \end{align*}
  Recall that the quantity $\Delta_{ia}(z^*,\Lambda)$ is defined in~\eqref{eq:def_deltaia} by  
 \begin{align*}
  \Delta(i,a) 
  & \weq \sup_{t \in (0,1)} (1-t) \dren_t \left( \bigotimes_{j \ne i} \Ber \left( \rho_n \lambda_{i z^*_j} \lambda_{j a} \right), \bigotimes_{j \ne i} \Ber \left( \rho_n \lambda_{i z^*_j} \lambda_{j z_i^*} \right) \right). 
 \end{align*}
  \cite[Lemma~2]{dreveton2024universal} shows that for all $a \ne z_i^*$ we have 
  \begin{align*}
    \pr \left( \phi^{\MLE}( A_{i\cdot} ) = a \cond z_i^* \right) \weq e^{ -(1+o(1)) \Delta(i,a)},
  \end{align*}
  provided that $\Delta(i,a) = \omega(1)$. Furthermore, if $\Delta(i,a) = \omega(\log k)$, union bounds imply that 
  \begin{align}
  \label{eq:lower_bound_MLE_error}
    \pr \left( \phi^{\MLE} \left( A_{i\cdot} \right) \ne z_i^* \right) \weq e^{-(1+o(1)) \Chernoff(i, a) }
  \end{align}
  where   
  \begin{align*}
   \Chernoff(i, z^*) \weq \min_{a \ne z_i^*} \, \Delta(i,a)
  \end{align*}
  is the Chernoff information associated with this hypothesis testing problem. 

\subsection{Lower-bounding the Optimal Error}

\begin{proof}[Proof of Theorem~\ref{thm:lower_bound}]
 For simplicity, we shorten $\cZ_n(\pi, \epsilon)$ by $\cZ$. Let $z^* \in \cZ$ be the true cluster membership vector. We denote the set of vertices in cluster $a$ by $\Gamma_a(z^*) = \{ i \in [n] \colon z^*_i = a \}$.  Following the same proof strategy as previous works on clustering block models~\citep{gao2018community,dreveton2024universal}, we define a clustering problem over a subset of $[k]^n$ to avoid the issues of label permutations in the definition of the loss function~\eqref{eq:def_haming_loss_function}. 
 For every cluster $a \in [k]$, 
 we define the set $T_a$ of the $|\Gamma_a(z^*)| - \frac{n (1-\epsilon) \pi_{\min}}{4k}$ vertices belonging to cluster~$a$ and having the largest $\Chernoff(i,z^*)$.  
 We motivate this as follows. A vertex $i$ with a large $\Chernoff(i,z^*)$ implies that if a genie provides~$z_{-i}^*$ (the community labels of all vertices but $i$), the inference of $z_i^*$ is easy. Hence, the set $T_a$ contains the vertices belonging to the cluster $a$ that are the easiest to cluster, and therefore a good estimator $\hz$ should correctly infer the vertices belonging to $T_a$. In contrast, vertices with small $\Chernoff(i,z^*)$ may be impossible to cluster, even with the best estimator, and these are the vertices that matter in deriving the lower-bound. 
 Let $T = \cup_{a\in[k]} T_a$ and define a new parameter space $\tcZ \subseteq \cZ$ 
 \[
  \tcZ \weq \{ z \colon z_i = z^*_i \text{ for all } i \in T \quad \text{ and } \quad \frac{ | \Gamma_a(z) | }{n} \in [(1-\epsilon)\pi_a, (1+\epsilon)\pi_a] \}. 
 \]
 This new space $\tcZ$ is composed of all vectors $z \in \cZ$ that only differ from~$z^*$ on the indices $i$'s that do not belong to $T$. By definition of $T$, these vertices are the hardest to cluster. By construction of $\tcZ$, we have for any $z, z' \in \tcZ$  
\begin{align*}
 \ham(z,z') \weq \sum_{i=1}^n \1\{z_i \ne z'_i\} \wle \left|T^c\right| \weq k \frac{ n (1-\epsilon) \pi_{\min} }{ 4k }.
\end{align*}
Because $z \in \tcZ \subset \cZ$, we have by definition of $\cZ$ that $\min_{a \in [k]} |\Gamma_a(z)| \ge (1-\epsilon) n \pi_{\min}$. Therefore, the previous inequality ensures that $\ham(z,z') < 2^{-1} \min_{a \in [k]} |\Gamma_a(z)|$ for all $z,z' \in \cZ$. We can thus apply Lemma~\ref{lemma:unique_permutation_minimiser} to establish that 
\begin{align}
\label{eq:loss_in_reduceSpace}
 \forall z, z' \in \tcZ \ \colon \quad \loss(z,z') \weq \frac1n \ham(z,z') \weq \frac1n \sum_{i \in T^c} \1\{ z_i \ne z'_i \}. 
\end{align}
For any estimator $\hz$, we can build an estimator $\hz' \in \tcZ$ such that 
\[
 \hz_i' \weq 
 \begin{cases}
  z_i^* & \text{ if } i \in T, \\
  \hz_i & \text{ otherwise,}   
 \end{cases}
\]
and this estimator satisfies $\loss(z^*,\hz') \le \loss(z^*,\hz)$.
Therefore, 
\begin{align*}
 \inf_{\hz \in \cZ} \E \, \loss(z^*, \hz) 
 \wge \inf_{\hz' \in \tcZ} \E \, \loss(z^*, \hz')
 \weq \frac1n \inf_{\hz' \in \tcZ} \E \, \ham(z^*, \hz),
\end{align*}
where the last equality follows from~\eqref{eq:loss_in_reduceSpace}. 
Hence, we obtain 
\begin{align*}
 \inf_{\hz \in \cZ} \E \, \loss(z^*, \hz) 
 \wge \frac1n \inf_{\hz} \sum_{i \in T^c} \pr \left( \hz_i \ne z_i^* \right)
 \wge \frac1n \sum_{i \in T^c} \inf_{\hz_i} \pr \left( \hz_i \ne z_i^* \right). 
\end{align*}
From Equation~\eqref{eq:lower_bound_MLE_error}, we have 
\begin{align*}
 \inf_{\hz_i} \pr \left( \hz_i \ne z_i^* \right) \wge e^{-(1+\eta_i) \Chernoff(i, z^*) },
\end{align*}
for some $\eta_i = o(1)$. Let $\eta = \max_i \eta_i$. We obtain 
\begin{align*}
 \inf_{\hz \in \cZ} \E \, \loss(z^*, \hz) 
 & \wge \frac{|T^c|}{n} \frac{1}{|T^c|} \sum_{i \in T^c} e^{-(1+\eta) \Chernoff(i, z^*)} \\
 & \wge \frac{|T^c|}{n} \frac1n \sum_{i \in [n]} e^{-(1+\eta) \Chernoff(i, z^*)} \\
 & \weq \frac{ (1-\epsilon) \pi_{\min} }{ 4 } \frac1n \sum_{i \in [n]} e^{-(1+\eta) \Chernoff(i, z^*)}, 
\end{align*}
where the second inequality uses the fact that $T^c$ collects the indices of the vertices with the smallest $\Chernoff(i, z^*)$, and the last line uses $\frac{|T^c|}{n} = \frac{ \alpha (1-\epsilon) \pi_{\min} }{ 4 }$ (by definition of $T$). 

Finally, note that we can always chose $\eta$ to be nonnegative and thus the function $x \mapsto x^{1+\eta}$ is convex. Hence, by Jensen's inequality, we have 
\begin{align*}
\frac1n \sum_{i \in [n]} \left( e^{-\Chernoff(i, z^*)} \right)^{1+\eta}
\wge \left( \frac1n \sum_{i \in [n]} e^{-\Chernoff(i, z^*)} \right)^{1+\eta}. 
\end{align*}

\end{proof}

\subsection{Additional Lemmas}

\begin{lemma}[Lemma~C.5 in~\cite{avrachenkov2020community}] 
\label{lemma:unique_permutation_minimiser}
 Let $z_1, z_2 \in [k]^n$ such that $\ham(z_1, \tau^* \circ z_2) < \frac12 \min_{a \in [k]} |\Gamma_a(z_1)|$ for some $\tau^* \in \Sym(k)$. Then $\tau^*$ is the unique minimizer of $\tau \in \Sym(k) \mapsto \ham(z_1, \tau \circ z_2)$. 
\end{lemma}

\section{Proof of Theorem~\ref{thm:upper_bound}}

\paragraph{Warm-up: notations and MLE}
Let $z \in [k]^n$ be any vertex labeling. We denote $L(z) = \pr \left( A \cond z \right)$ the likelihood of $z$ given the observation $A$. We study the performance of the maximum likelihood estimator $\hz = \hz(A)$ defined by
\begin{align*}
 \hz \weq \argmax_{ z \in [k]^n } L(z), 
\end{align*}
where ties are broken arbitrarily. Hence, by definition, the MLE is any estimator $\hz$ such that
\begin{align*}
 L( \hz ) \wge L(z) \quad \text{ for all } z \in [k]^n. 
\end{align*}

Moreover, we have 
\begin{align*}
 n \, \E \left[ \loss(z^*, \hz) \right] 
 \weq \E \left[ \min_{ \sigma \in \Sym([k]) } \ham( z^*, \sigma \circ z) \right]  
 \wle \E \left[ \ham( z^*, z) \right]. 
\end{align*}
We also recall (see~\cite[Lemma~7]{dreveton2023exact}) that, for any $z, z' \in [k]^n$ we have  
\begin{align*}
 \loss(z, z') \wle n(1 - 1 / k ).
\end{align*}
Therefore,  
\begin{align*}
 n \E \left[ \loss( z^*, z) \right]
 \wle \sum_{ m = 1}^{n(1 - 1 / k )} m \pr \left( \ham( z^*, \hz ) = m \right). 
\end{align*}
For technical reasons that will become clear in the end of the proof, we first need to split the sum into two parts. Let $m_0 \ge 1$, whose value will be determined later. We have
\begin{align*}
 \E \left[ \ham( z^*, z) \right]
 & \weq \sum_{ m = 1}^{m_0} m \pr \left( \ham( z^*, \hz ) = m \right) + \sum_{ m = m_0+1}^{n(1 - 1 / k )} m \pr \left( \ham( z^*, \hz ) = m \right) \\
 & \wle m_0 + \sum_{ m = m_0+1}^{n(1 - 1 / k )} m \pr \left( \ham( z^*, \hz ) = m \right).
\end{align*}

Let us denote $\cZ_m$ the set of vertex labeling $z \in [k]^n$ such that $\ham( z^*, z ) = m$. By definition of the maximum likelihood and by union bounds, we have 
\begin{align*}
 \pr \left( \hz \in \cZ_m \right)
 \wle \pr \left( \exists z \in \cZ_m \colon L(z) \ge L(z^*) \right) 
 \wle \sum_{ z \in \cZ_m } \pr \left( L(z) \ge L(z^*) \right). 
\end{align*}

Hence, by combining the previous inequalities, we obtain
\begin{align}
\label{eq:in_proof_bound_loss_mle}
 \E \left[ \loss(z^*, \hz) \right] 
 \wle \frac1n \left( m_0 + \sum_{ m = m_0+1}^{n(1 - 1 / k )} m \sum_{ z \in \cZ_m } \pr \left( L(z) \ge L(z^*) \right) \right).
\end{align}

A large part of the rest of the proof is devoted to upper-bound $\sum_{ z \in \cZ_m } \pr \left( L(z) \ge L(z^*) \right)$ for an arbitrary $m$. We first observe that 
\begin{align*}
 L(z) 
 & \weq \prod_{i < j } \pr \left( A_{ij} \cond z_i, z_j \right) \\ 
 & \weq \prod_{i < j } \left( \rho_n \lambda_{i z_j} \lambda_{j z_i} B_{z_i z_j} \right)^{A_{ij}} \left( 1 - \rho_n \lambda_{i z_j} \lambda_{j z_i} B_{z_i z_j} \right)^{1-A_{ij}}.
\end{align*}
In all the following, to avoid overburdening the notations, we denote $P_{ij}^z = \rho_n \lambda_{i z_j} \lambda_{j z_i} B_{z_i z_j}$. 
We also introduce 
\begin{align*}
 \Gamma(z,z^*) \weq \left\{ (i,j) \colon 1 \le i \ne j \le n \text{ and } (z_i,z_j) \ne (z_i^*,z_j^*) \right\}
\end{align*}
We have 
\begin{align*}
 \frac{ L(z) }{ L(z^*) } 
 \weq \prod_{ \substack{ i < j \\ (i,j) \in \Gamma(z,z^*) } } \left( \frac{ P_{ij}^z }{ P_{ij}^* } \right)^{A_{ij}} \left( \frac{ 1-P_{ij}^z }{ 1-P_{ij}^* } \right)^{1-A_{ij}}. 
\end{align*}

Therefore, by Chernoff bounds, we have for any $t > 0$, 
\begin{align}
 \pr \left( L(z) > L(z^*) \right) 
 & \weq \pr \left( e^{ t \log \frac{L(z)}{L(z^*)}  } > 1 \right) \nonumber \\
 & \wle \prod_{ \substack{ i < j \\ (i,j) \in \Gamma(z,z^*) } } \E \left[ e^{ t \left( A_{ij} \log \frac{ P_{ij}^z }{ P_{ij}^* } + (1-A_{ij}) \log \frac{ 1-P_{ij}^z }{ 1-P_{ij}^* } \right) } \right] \nonumber \\
 & \weq \prod_{ \substack{ i < j \\ (i,j) \in \Gamma(z,z^*) } } e^{ -(1-t) \dren_t\left( P_{ij}^z, P_{ij}^* \right) }.
 \label{eq:in_proof_upper_bound_chernoff}
\end{align}

For ease of the exposition, we start by deriving an upper bound on $\sum_{ z \in \cZ_m } \pr \left( L(z) \ge L(z^*) \right)$ in the simplest case $m=1$. We do the general case $m \ge 1$ later.  

\paragraph{(i) Case $m=1$.} 
Observe that 
$$
 \cZ_1 \weq \{ z \in [k]^n \colon \ham(z,z^*) = 1 \} \weq \{ \tz^{ua}, u \in [n], a \in [k] \setminus \{ z_u^* \} \},
$$ 
where $\tz^{ua}_v = z_v^*$ for all $u \ne v$ and $\tz^{ua}_u = a$. Hence, 
\begin{align}
 \pr \left( \ham(z^*, \hz ) = 1 \right) 
 & \weq \pr \left( \exists u \in [n], \ \exists a \in [k] \setminus \{ z_u^* \} \colon  L(\tz^{ua} ) > L(z^*) \right) \nonumber \\
 & \wle \sum_{u=1}^n \sum_{a \in [k] \setminus \{ z_u^* \} } \pr \left( L( \tz^{ua} ) > L(z^*) \right). \label{eq:in_proof_union_bound_m1}
\end{align}
 Moreover, for any $u \in [n]$ and $a \in [k] \setminus \{z_u^*\}$, we have 
\begin{align*}
 \pr \left( L( {\tz^{ua}} ) > L(z^*) \right) 
 & \wle e^{ - (1-t) \sum_{j \ne u} \dren_t \left(  P_{uj}^{ \tz^{ua} }, P_{uj}^* \right) }.
\end{align*}
 This last inequality is valid for any $t > 0$. Applying it with $t^* = \argmax_{t \in (0,1)} (1-t) \sum_{j \ne u} \dren_t \left(  P_{uj}^{\tz^u}, P_{uj}^* \right)$, we obtain 
\begin{align*}
 \pr \left( L(\tz^{ua}) > L(z^*) \right) \wle e^{ - \Delta_{u a} } 
 \wle e^{ - \Chernoff(u, z^*) },
\end{align*}
because $\Chernoff(u, z^*) = \min_{a \ne z_u^*} \Delta_{u a}$. Hence, using~\eqref{eq:in_proof_union_bound_m1} we have 
\begin{align*}
 \pr \left( \ham(z^*, \hz ) = 1 \right) 
 \wle k \sum_{u=1}^n e^{ - \Chernoff(u, z^*) }.
\end{align*}

\paragraph{(ii) Case $m \ge 2$} Consider now $z$ such that $\ham(z,z^*) = m$. Introduce $u_1, \cdots, u_m$ the $m \ge 2$ vertices satisfying $z_{u_p} \ne z^*_{u_p}$ for all $p \in [m]$. By definition, for any $v \not\in \{u_1, \cdots, u_p\}$, we have $z_v = z_v^*$. 

Observe that
\begin{align*}
  \Gamma(z,z^*) & \weq 
  \{ (i,j) \colon i \ne j \text{ and } (z_i,z_j) \ne(z_i^*, z_j^*) \} 
  \weq S_1 \cup S_2, 
\end{align*}
where 
\begin{align*}
 S_1 & \weq \{ (i,j) \colon z_i \ne z_i^* \text{ and } j \ne i \} \\
 S_2 & \weq \{ (i,j) \colon z_i = z_i^* \text{ and } z_j \ne z_j^* \}. 
\end{align*}
Thus, we have 
\begin{align*}
 \sum_{ \substack{ i < j \\ (i,j) \in \Gamma(z,z^*) } } (1-t) \dren_t\left( P_{ij}^z, P_{ij}^* \right) 
 & \weq \frac12 \sum_{ \substack{ (i,j) \in \Gamma(z,z^*) } } (1-t) \dren_t\left( P_{ij}^z, P_{ij}^* \right) \\
 & \weq \frac12 \left( \underbrace{ \sum_{ (i,j) \in S_1 } (1-t) \dren_t\left( P_{ij}^z, P_{ij}^* \right) }_{T_1(t)} + \underbrace{ \sum_{ (i,j) \in S_2 } (1-t) \dren_t\left( P_{ij}^z, P_{ij}^* \right) }_{ T_2(t) } \right).
\end{align*}
Notice further that
\begin{align}
\label{eq:in_proof_def_T_1}
 T_1(t) \weq \sum_{i \in \{u_1, \cdots, u_m\} } \sum_{j \ne i} (1-t) \dren_t\left( P_{ij}^z, P_{ij}^* \right)
\end{align}
and 
\begin{align*}
 T_2(t) & \weq \sum_{i \notin \{u_1, \cdots, u_m \} } \sum_{j \in \{u_1, \cdots, u_m \} } (1-t) \dren_t\left( P_{ij}^z, P_{ij}^* \right) \\
 & \weq \sum_{i \in \{u_1, \cdots, u_m \} } \sum_{j \notin \{u_1, \cdots, u_m \} } (1-t) \dren_t\left( P_{ij}^z, P_{ij}^* \right) \\ 
 & \weq T_1(t) - \sum_{i \in \{u_1, \cdots, u_m \} } \sum_{ \substack{ j \in \{u_1, \cdots, u_m \} \\ j \ne i } } (1-t) \dren_t\left( P_{ij}^z, P_{ij}^* \right).
\end{align*}

Combined to the Chernoff bounds~\eqref{eq:in_proof_upper_bound_chernoff}, this leads
\begin{align}
\label{eq:in_proof_upper_bound_chernoff_with_T}
 \pr \left( L(z) \ge L(z^*) \right) \wle e^{ - T_1(t) + T_3(t) },
\end{align}
where $T_3(t)$ is given by 
\begin{align}
\label{eq:in_proof_def_T_3}
 T_3(t) \weq \frac12 \sum_{i \in \{u_1, \cdots, u_m \} } \sum_{ \substack{ j \in \{u_1, \cdots, u_m \} \\ j \ne i } } (1-t) \dren_t\left( P_{ij}^z, P_{ij}^* \right). 
\end{align}

Let us lower-bound $T_1$. For $p \in [m]$, denote $t_p = \argmax_{t \in (0,1)} (1-t) \sum_{j \ne u} \dren_t \left(  P_{u_p j}^z, P_{u_p j}^* \right)$. Note that $t_p$ is bounded away from one, as when $t=1$, the objective function inside the argmax equals 0. We also recall that, for any $\alpha, \beta \in (0,1)$ with $\alpha \le \beta$, and any probability distributions $f$ and $g$, we have~\cite[Theorem~16]{van2014renyi}
\begin{align*}
 \frac{\alpha}{\beta} \frac{1-\beta}{1-\alpha} \dren_{\beta}(f,g) \wle \dren_{\alpha} (f,g) \wle \dren_{\beta} (f,g). 
\end{align*}
Denote $t^* = \min \{ t_1, \cdots, t_m \}$. Without loss of generality, suppose that $t^* = t_1$. Using the previous inequality with $\alpha = t_p$ and $\beta = t_1$, we have, 
\begin{align*}
\sum_{ j \in [n] \backslash \{ u_p \} } \dren_{t_1} \left( P_{u_p j}^z, P_{u_p j}^* \right) 
\wge \sum_{ j \in [n] \backslash \{ u_p \} } \dren_{t_p} \left( P_{u_p j}^z, P_{u_p j}^* \right),
\end{align*}
for any $p \in [m]$. Thus, 
\begin{align*}
(1-t_1) \sum_{ j \in [n] \backslash \{ u_p \} } \dren_{t_1} \left( P_{u_p j}^z, P_{u_p j}^* \right)
& \wge \frac{1-t_1}{1-t_p} (1-t_p) \sum_{ j \in [n] \backslash \{ u_p \} } \dren_{t_p} \left( P_{u_p j}^z, P_{u_p j}^* \right) \\
& \weq \frac{1-t_1}{1-t_p} \Chernoff(u_p), 
\end{align*}
by definition of $t_p$. Because all the $t_p$ are bounded away from 1 and $t_1 = \min\{t_1, \cdots, t_m\}$, we have $\frac{1-t_1}{1-t_p} \ge C$ for some constant $C \ge 1$. Recalling the definition of $T_1$ in~\eqref{eq:in_proof_def_T_1}, we obtain
\begin{align}
 T_1(t_1) 
 & \wge \sum_{p =1 }^m \frac{1-t_1}{1-t_p} \Chernoff(u_p) \nonumber \\
 & \weq \Chernoff(u_1) + \sum_{p=2}^m \frac{1-t_1}{1-t_p} \Chernoff(u_p) \nonumber \\
 & \wge \Chernoff(u_1) + C \sum_{p=2}^m \Chernoff(u_p). 
  \label{eq:in_proof_bound_T1}
\end{align}

We now upper-bound $T_3(t_1)$, defined in~\eqref{eq:in_proof_def_T_3}. By Assumption~\ref{assumption:scaling_parameters}, all the \Renyi divergences are of the same order. Thus, there exists a quantity $C_m'$ such that $C_n' = 1$ and 
\begin{align}
 T_3(t_1) 
 & \wle \frac12 \sum_{i \in \{u_1, \cdots, u_m \} } C'_m \frac{m-1}{n} \sum_{ j \ne i } (1-t_1) \dren_{t_1}\left( P_{ij}^z, P_{ij}^* \right) \nonumber \\
 & \wle C'_m \frac{m}{2n} \sum_{p \in \{1, \cdots, m \} } \Chernoff(u_p,z^*).
 \label{eq:in_proof_bound_T3}
\end{align}
In the rest of the proof, we denote $\delta_m = C'_m \frac{m}{2n}$. 

By combining~\eqref{eq:in_proof_bound_T1} and~\eqref{eq:in_proof_bound_T3} with the Chernoff bound~\eqref{eq:in_proof_upper_bound_chernoff_with_T}, we have 
\begin{align*}
 \pr \left( L(z) \ge L(z^*) \right) 
 \wle e^{ - \Chernoff(u_1) } e^{ - (C - \delta_m / 2) \sum_{p=2}^m \Chernoff(u_p)}. 
\end{align*}

We recall that $z \in \cZ_m$ if and only if there exists a set $\{u_1, \cdots, u_m\}$ of $m$ distinct vertices such that 
\begin{align*}
  z_u = z^*_u \iff u \notin \{u_1, \cdots, u_m\}.
\end{align*}
 Moreover, for any such set $\{u_1, \cdots, u_m\}$, there exists $(k-1)^m$ ways to construct a $z \in \cZ_m$. Hence, 
\begin{align*}
 \sum_{ z \in \cZ_m } \pr \left( L(z) \ge L(z^*) \right) 
 & \wle (k-1)^m \sum_{ \{ u_1, \cdots, u_m \} } e^{ - \Chernoff(u_1) } e^{ -  (C-\delta_m) \sum_{p=2}^m \Chernoff(u_p) } \\
 & \weq (k-1)^m \sum_{u_1 = 1}^n e^{ - \Chernoff(u_1) } \sum_{ \{u_2, \cdots, u_m\} } e^{ -  (C - \delta_m) \sum_{p=2}^m \Chernoff(u_p) }.
\end{align*}
In the previous inequality, the second summation is over all set $\{u_2, \cdots, u_m\}$ of $m-1$ elements belonging to $[n] \backslash \{u_1\}$. There are 
\begin{align*}
 \binom{n-1}{m-1} \wle \left( \frac{e (n-1)}{ m-1} \right)^{m-1}
\end{align*}
ways of choosing such set. We finally obtain
\begin{align*}
 \sum_{ z \in \cZ_m } \pr \left( L(z) \ge L(z^*) \right) 
 &\wle \sum_{u_1 = 1}^n e^{ - \Chernoff(u_1) }  \left( \frac{e (k-1) (n-1)}{ m-1 } e^{ -  (C - \delta_m) \min_{i \in [n] }\Chernoff(i) } \right)^{m-1} \\
 &\wle \sum_{u_1 = 1}^n e^{ - \Chernoff(u_1) }  \left( \frac{e k n}{ m-1 } e^{ - (C-\delta_m) \min_{i \in [n] }\Chernoff(i) } \right)^{m-1}.
\end{align*}

\paragraph{Ending the proof.}
Going back to~\eqref{eq:in_proof_bound_loss_mle}, we have
\begin{align}
\label{eq:in_proof_bound_loss_mle_end}
 \E \left[ \loss(z^*, \hz) \right] 
 \wle \frac1n \left( m_0 + \sum_{u_1 = 1}^{n} e^{ - \Chernoff(u_1) } \sum_{ m = m_0+1 }^{n(1-1/k)} m Q_m^{m-1} \right), 
\end{align}
where $Q_m = \frac{e k n}{ m-1 } e^{ - (C-\delta_m) \min_{i \in [n] }\Chernoff(i) }$. 
Denote also $R = \sum_{u_1 = 1}^n e^{ - \Chernoff(u_1) }$ and $ B = 2enk e^{ -C \min_i \Chernoff(i,z^*)}$, and recall $C \ge 1$. We also introduce 
$$m_1 = \lfloor 2enk e^{- (C - \delta_{n(1-1/k)} ) \min_i \Chernoff(i,z^*)} \rfloor.$$
By assumption, we have $\delta_{n(1-1/k)} < 1-\epsilon$ and thus $ m_1 = o(n)$. 
Observe that
\begin{align*}
 Q_m \wle \frac12 \quad \forall m \in \{m_1+1, \cdots, n(1-1/k) \} 
\end{align*}
and thus
\begin{align}
\label{eq:in_proof_sum_last_terms}
 \sum_{ m = m_1+1 }^{n(1-1/k)} m Q_m^{m-1} \wle \sum_{ m = m_1+1 }^{\infty} m \left( \frac12 \right)^{m-1} \wle 4 \frac{m_1 + 1}{2^{m_1}} \wle 4 
\end{align}
by using properties on geometric sums (see Lemma~\ref{the:PowerSeriesBound}). 

We still need to upper-bound $\sum_{ m = m_0+1 }^{m_1} m Q_m^{m-1}$. Let $\tm_0 = 2 ekR e^{ C \delta_{m_1} \min_i \Chernoff(i,z^*) }$. Observe that, for any $\tm_0 \le m \le m_1$, we have $Q_m \le 1/2$. Then, we are left with two cases. 

(a) If $\tm_0 \le 1$, then chose $m_0 = 0$. Then, we simply have 
\begin{align*}
 \sum_{ m = m_0 }^{m_1} m Q_m^{m-1} \weq \sum_{ m = 1 }^n m Q_m^{m-1} \wle \sum_{ m = 1 }^{\infty} m \left( \frac12 \right)^{m-1} \wle 4, 
\end{align*}
by using Lemma~\ref{the:PowerSeriesBound} as above. By combining~\eqref{eq:in_proof_bound_loss_mle_end} and~\eqref{eq:in_proof_sum_last_terms}, we have 
\begin{align*}
  \E \left[ \loss(z^*, \hz) \right] 
 \wle 8 \frac{R}{n}.
\end{align*}

(b) Otherwise, chose $m_0 = \lceil \tm_0 \rceil$. Then, we upper-bound $\sum_{ m = m_0 }^{m_1} m Q_m^{m-1}$ by $4$ as above, and we obtain from~\eqref{eq:in_proof_bound_loss_mle_end} that 
\begin{align*}
  \E \left[ \loss(z^*, \hz) \right] 
 \wle \frac1n \left( m_0 + 8 R \right)
\end{align*}
Moreover, $m_0 = \lceil \tm_0 \rceil \le 2 \tm_0$. This gives
\begin{align*}
  \E \left[ \loss(z^*, \hz) \right] 
 \wle \frac1n \left( 2 \tm_0 + 8 R \right)
 \wle \frac{R}{n} \left( 2 ek e^{ C \delta_{m_1} \min_i \Chernoff(i,z^*) } + 8 \right).
\end{align*}
Observe that this last upper-bound is also an upper-bound for $\E \left[ \loss(z^*, \hz) \right] $ in the case (a).
To finish the proof, we recall that $m_1 = o(n)$ and thus $\delta_{m_1} = o(1)$ by definition of $\delta_m$. 
\qed

\paragraph{Additional Lemma}

This lemma and its proof are taken from~\cite[Lemma~A.8]{avrachenkov2020community}, and reproduced here for the sake of completeness. 
\begin{lemma}
\label{the:PowerSeriesBound}
For any integer $M \ge 1$ and any number $0 \le s < 1$,
\[
 M s^{M}
 \wle \sum_{m=M}^\infty m s^{m}
 \wle (1-s)^{-2} M s^M.
\]
\end{lemma}
\begin{proof}
Denote $S = \sum_{m=M}^\infty m s^{m}$. By differentiating 
$
 \sum_{m=M}^\infty s^{m}
 = (1-s)^{-1} s^M
$ with respect to $s$, we find that
\begin{align*}
 s^{-1} S
 \weq \sum_{m=M}^\infty m s^{m-1}
 \weq (1-s)^{-2} s^M + (1-s)^{-1} M s^{M-1},
\end{align*}
from which we see that
\begin{align*}
 S
 \weq s (1-s)^{-2} \Big( s^M + (1-s) M s^{M-1} \Big)
 \weq \frac{M s^M}{(1-s)^2} \Big( 1 - s(1-1/M)  \Big)
\end{align*}
The upper bound now follows from $1 - s(1-1/M) \le 1$.
The lower bound is immediate, corresponding to the first term of the nonnegative series.
\end{proof}

\section{Proof of Proposition~\ref{prop:recovery_rate_pabm} and Examples~\ref{example:recovery_pabm_degree} and~\ref{example:recovery_pabm_xi}}

\subsection{Chernoff divergence for Homogeneous PABM}

We start with the following lemma.
\begin{lemma}
\label{lemma:chernoff_homogeneous_pabm}
Consider a PABM with homogeneous interactions, and $k$ equal-size communities. Suppose the coefficients $\lambdain_1, \cdots, \lambdain_n$ (resp., $\lambdaout_1, \cdots, \lambdaout_n$) are sampled iid from a distribution $\Din$ (resp., $\Dout$), where $\Din$ and $\Dout$ are two distributions supported on $\R_+$ and with mean 1. Let $i \in [n]$. We have 
\begin{align*}
 \Chernoff(i,z^*) \weq (1+o(1)) \frac{n \rho_n}{k} \E \left[ \left( \sqrt{p_0 \lambdain_i Y} - \sqrt{q_0 \lambdaout_i Y'} \right)^2 \right],
\end{align*}
where $Y$ and $Y'$ are two independent random variables sampled from $\Din$ and $\Dout$, respectively. 
\end{lemma}

\begin{proof}
 We apply the law of large number to the quantity $\delta$ defined in the equation above Proposition~\ref{prop:recovery_rate_pabm}. 
\end{proof}

\begin{lemma}
\label{lemma:optimal_rate_homogeneous_pabm}
Consider the same setting and notations as in Lemma~\ref{lemma:chernoff_homogeneous_pabm}. 
We also suppose that the distributions $\Din$ and $\Dout$ have pdf $f_{\Din}$ and $f_{\Dout}$ with respect to the Lebesgue measure. Denote $\gammain = \E[\sqrt{Y}]$ and $\gammaout = \E[\sqrt{Y'}]$, where $Y \sim \Din$ and $Y' \sim \Dout$. Finally, suppose that $p_0 > 0$ and let $\xi = q_0 / p_0$. We have 
\begin{align*}
 \frac{1}{n} \sum_{i=1}^n \exp\left( - \frac{n \rho_n}{k} p_0  \left( 1 + \lambdain_i - 2 \gamma \sqrt{\lambdain_i} \right) \right) \weq (1+o(1)) J_n,
\end{align*}
where 
\begin{align*}
 J_n \weq \int \int \exp \left( - \frac{n \rho_n}{k} p_0 \left( x + \xi y - 2 \gammain \gammaout \sqrt{\xi} \sqrt{ xy} \right) \right) f_{\Din}(x) f_{\Dout}(y) \diff x \diff y. 
\end{align*}
\end{lemma}

\begin{proof}
 From Lemma~\ref{lemma:chernoff_homogeneous_pabm}, we have 
\begin{align*}
 \Chernoff(i,z^*) 
 & \weq (1+o(1)) \frac{n \rho_n}{k} p_0 \E \left[ \left( \sqrt{\lambdain_i Y} - \sqrt{\xi \lambdaout_i Y'} \right)^2 \right] \\
 & \weq (1+o(1)) \frac{n \rho_n}{k} p_0 \left( \lambdain_i + \xi \lambdaout_i - 2 \gammain \gammaout \sqrt{\xi} \sqrt{ \lambdain_i \lambdaout_i} \right).
\end{align*}
As $\lambdain_i \sim \Din$ and $\lambdaout_i \sim \Dout$, computing $\frac{1}{n} \sum_i e^{-\Chernoff(i,z^*) }$ resumes to compute 
\begin{align*}
\lim_{i\to\infty} \frac{1}{n} \sum_{i=1}^n \exp \left( - \frac{n \rho_n}{k} p_0 \left( x + \xi y - 2 \gammain \gammaout \sqrt{\xi} \sqrt{ xy} \right) \right). 
\end{align*}

In particular, $\exp \left( - \frac{n \rho_n}{k} p_0 \left( x + \xi y - 2 \gammain \gammaout \sqrt{\xi} \sqrt{ xy} \right) \right)$ is bounded by $[0,1],$ hence its variance is also upper bounded by $1.$ Let $J_n$ be the expectation of this quantity over $\Din, \Dout.$ Centering the variable, we bound the total variance
\begin{align*}
    \sum_{i=1}^n n^{-2}\Var\left(\exp \left( - \frac{n \rho_n}{k} p_0 \left( x + \xi y - 2 \gammain \gammaout \sqrt{\xi} \sqrt{ xy} \right) \right)-J_n \right) \ < \ \sum_i n^{-2} \ < \ \infty.
\end{align*}
Kolmogorov's variance criterion for averages~\citep[Lemma~5.22]{Kallenberg} implies
\begin{align*}
    \frac{1}{n}\sum_{i=1}^n \left(\exp \left( - \frac{n \rho_n}{k} p_0 \left( x + \xi y - 2 \gammain \gammaout \sqrt{\xi} \sqrt{ xy} \right) \right)-J_n \right) \asto 0.
\end{align*}
Therefore the limit converges to its expectation almost surely, 
\begin{align*}
\lim_{i\to\infty} \frac{1}{n} \sum_{i=1}^n \exp \left( - \frac{n \rho_n}{k} p_0 \left( x + \xi y - 2 \gammain \gammaout \sqrt{\xi} \sqrt{ xy} \right) \right) \weq J_n,
\end{align*}
where
\begin{align*}
 J_n \weq \int \int \exp \left( - \frac{n \rho_n}{k} p_0 \left( x + \xi y - 2 \gammain \gammaout \sqrt{\xi} \sqrt{ xy} \right) \right) f_{\Din}(x) f_{\Dout}(y) \diff x \diff y. 
\end{align*}
\end{proof}

\subsection{Proof of Proposition~\ref{prop:recovery_rate_pabm}}

To prove Proposition~~\ref{prop:recovery_rate_pabm}, we apply Lemma~\ref{lemma:optimal_rate_homogeneous_pabm} in the particular case where $\Din$ is the uniform distribution $\Uni(1-c, 1+c)$ and $\Dout$ is the Dirac distribution at $1$. Hence, the integral $J_n$ given in Lemma~\ref{lemma:optimal_rate_homogeneous_pabm} becomes 
\begin{align*}
 J_n \weq \int \exp \left( - \frac{n \rho_n}{k} p_0 \left( x + \xi - 2 \gammain \sqrt{\xi} \sqrt{x} \right) \right) f_{\Din}(x) \diff x, 
\end{align*}
where $f_{\Din}(x) = \frac{1}{2c} \1(x \in (1-c,1+c))$, and the lower and upper limits of the integral are $1-c$ and $1+c,$ respectively. For simplicity we write $y=\sqrt{M x}$ where $M=n\rho_n p_0 / k$. We perform the following change of variable: $\sqrt{x} = \frac{y}{\sqrt{M}},$ $\diff y=\frac{1}{2}\sqrt{M}x^{-1/2}\diff x$ and $\diff x = \frac{2\sqrt{x}}{\sqrt{M}}\diff y=\frac{2y}{M}\diff y.$ The lower and upper integration limits become $y_-=\sqrt{M}\sqrt{1-c}$ and $y_+=\sqrt{M}\sqrt{1+c}.$ 
Changing variables and completing the square gets us 
\begin{align*}
 J_n & \weq \frac{1}{2c}\int_{y_-}^{y_+} \exp \left( -y^2-M \xi+2\gamma \sqrt{M}y\sqrt{\xi}\right) \frac{2y}{M}\diff y\\
 & \weq \frac{1}{2c}\int_{y_-}^{y_+} \exp \left( -(y - \gamma\sqrt{\xi} \sqrt{M})^2 +  M\xi(\gammain^2 -1)\right) \frac{2y}{M}\diff y\\
 & \weq  \frac{\exp (M\xi(\gammain^2 -1))}{cM} \int_{y_-}^{y_+} \exp \left( -(y - \gammain \sqrt{\xi}\sqrt{M})^2 \right) y \diff y.
\end{align*}
Again, substitute $u=y-\gammain \sqrt{\xi}\sqrt{M}$ to get 
\begin{align*}
 J_n \weq &  \frac{\exp (M\xi(\gammain^2 -1))}{cM} \int_{u_-}^{u_+} \exp \left( -u^2\right) \left(u + \gammain \sqrt{\xi}\sqrt{M}\right) \diff u\\
 \weq & \frac{\exp (M\xi(\gammain^2 -1))}{cM} \left(\int_{u_-}^{u_+} \exp \left( -u^2\right) u \diff u+ \gammain \sqrt{\xi} \sqrt{M}\int_{u_-}^{u_+} \exp \left( -u^2\right)  \diff u\right) ,
\end{align*} where $u_-=\sqrt{M}(\sqrt{1-c}-\gammain \sqrt{\xi})$ and $u_+=\sqrt{M}(\sqrt{1+c}-\gammain \sqrt{\xi}).$ The first integral can by solved  by-parts, and the later we recognize as the Gauss error function. Hence,
\begin{equation*}
 J_n \weq \frac{\exp (M\xi(\gammain^2-1))}{cM}\left(\frac{1}{2}(\exp(-u_+^2)-\exp(-u_-^2))+\frac{1}{2} \gammain \sqrt{\xi M}\sqrt{\pi}(\erf(u_+)-\erf(u_-))\right),
\end{equation*}
where $\erf(t) = 2 / \sqrt{\pi} \int_0^t e^{-t^2} \diff t.$ Moreover, the quantity $\gammain = \E_{Y \sim \Din} [ \sqrt{Y} ]$ can be computed explicitly. We obtain $\gammain = \frac{1}{2c} \int_{1-c}^{1+c} \sqrt{x} \diff x = \frac{1}{3c} \left( (1+c)^{\frac{3}{2}} - (1-c)^{\frac{3}{2}} \right)$. We denote this last quantity by $\gamma_c$, to emphasize that it depends only on $c$.
\qed

\subsection{Discussion Relative to Examples~\ref{example:recovery_pabm_degree} and~\ref{example:recovery_pabm_xi}}
\label{appendix:discussion_examples}

This involved expression of $J_n$ computed in Proposition~\ref{prop:recovery_rate_pabm} is well-behaved and practically interesting for particular values of $\xi$ and $c$. As such, a few remarks are in order. 

First (resuming Example~\ref{example:recovery_pabm_degree}), when $\xi=1$, we have $u_-=\sqrt{1-c}-\gamma<0$ and $u_+=\sqrt{1+c}-\gamma>0$ for all $c \in (0,1]$. Moreover, $M\to\infty$ as $n\to\infty,$ and 
 $u_{\mp}\to \mp\infty.$ So $J$ simplifies to the much simpler expression
\begin{align*}
  J_n \weq \frac{\gamma_c}{c} \sqrt{ \frac{k \pi}{ n \rho_n p_0}} \ \exp\left( - \frac{n\rho_n}{k}p_0 (1-\gamma_c^2) \right),
\end{align*} 
which only depends on $c$ (recall $\gamma_c = \frac{1}{3c} \left( (1+c)^{\frac{3}{2}} - (1-c)^{\frac{3}{2}} \right)$) and is monotonically decreasing over $c\in (0,1]$. This agrees with the following intuitive fact: as $c$ increases, the higher variance in the popularity heterogeneity aids recovery. 

Another interesting case is when we fix $c$ as in Example~\ref{example:recovery_pabm_xi}. As $\xi$ increases from $0$ to $1,$ $J$ first monotonically increase, then monotonically decrease and approaches $0$ as $\xi\to 1.$ In particular, $\xi=0$ corresponds to disconnected communities, hence clustering is trivial. As $\xi$ increases, the additional inter-cluster edges act as noise to our classification task. On the other hand, a very large $\xi$ allows us to better learn from the popularity patterns as $q_0$ gets closer and closer to $p_0,$ and leverage from the variance introduced by $c.$ Especially, $u_{\mp}\to \mp\infty$ when $\xi > \xi_0$ for some constant $c_0\in (0,1).$ Hence in this regime, classification is easy, as $J_n = \frac{\gamma_c\sqrt{\xi\pi}\exp (-M\xi(1-\gamma_c^2))}{c\sqrt{M}} \to 0$ as $M\to \infty$ and $\gamma_c^2-1 < 0$. This phenomena illustrates an interesting duality of the role of inter-cluster edges\textemdash they act as noise below a threshold $\xi_0,$ yet serves to emphasize the popularity variance introduced by $c$ above the same threshold.

To better illustrate this two phenomenon, we plot in Figure~\ref{fig:optimal_error_rate_uniform_ones} the error rates obtained for homogeneous PABM where $\lambdaout_i = 1$ and the $\lambdain$ are sampled from $\Uni(1-c,1+c)$. This illustrate the phenomenon highlighted by the Examples~\ref{example:recovery_pabm_degree} and~\ref{example:recovery_pabm_xi}: (i) the error rate do not vanish when the edge-density signal disappear and (ii) the error rate is not monotonously decreasing with the edge-density signal. 

 \begin{figure}[!ht]
 \centering
 \begin{subfigure}[b]{0.4\textwidth}
  \includegraphics[width=1.0\linewidth]{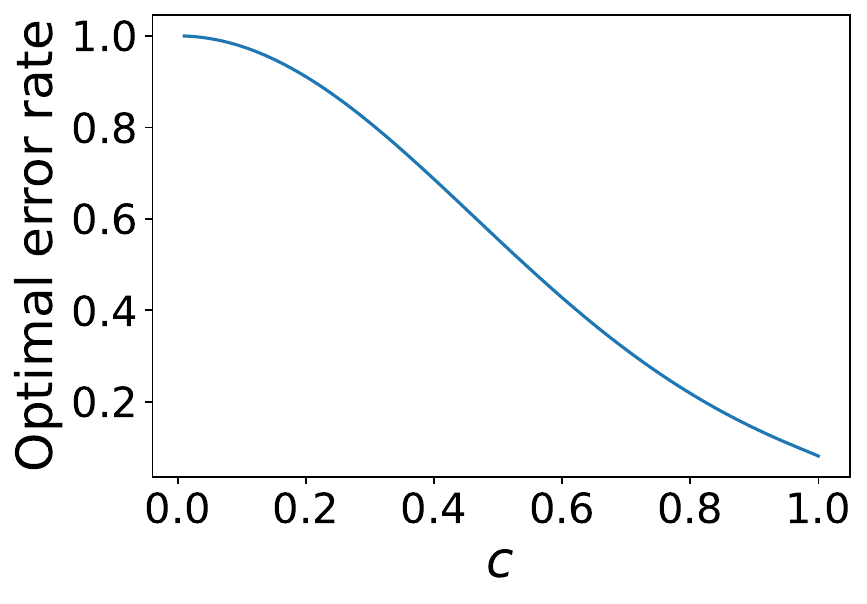}
  \caption{$\xi = 1$}
  \label{fig:optimal_error_rate_uniform_ones_varying_c}
 \end{subfigure}
 \hfill 
 \begin{subfigure}[b]{0.4\textwidth}
  \includegraphics[width=1.0\linewidth]{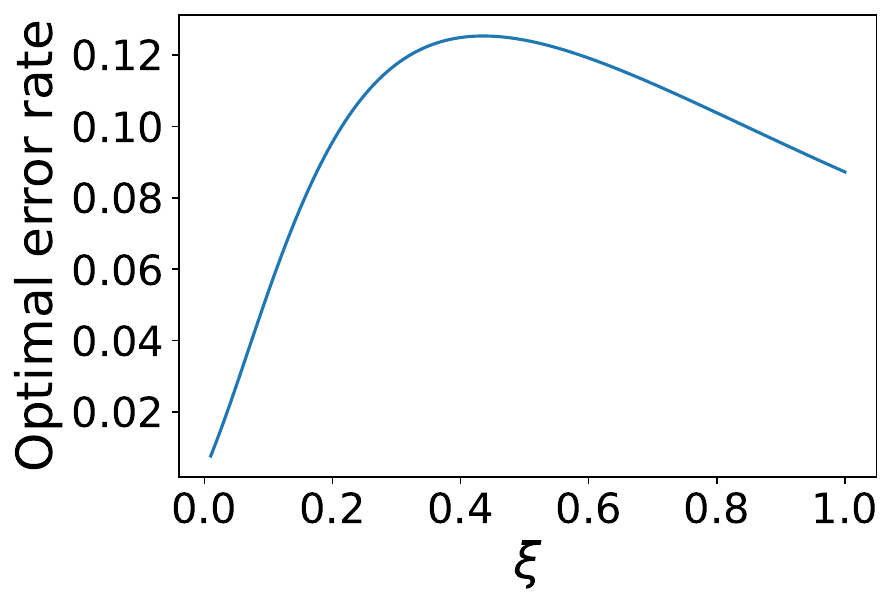}
  \caption{$c = 0.8$}
  \label{fig:optimal_error_rate_uniform_ones_varying_xi}
 \end{subfigure}
 \caption{Optimal error rate on PABM with homogeneous interactions. The matrix $P$ is given in Equation~\eqref{eq:experiments_homogeneous_PABM}, and we let $n=900$ vertices, $k=3$ clusters of same size, average edge density $\rho = 0.05$, and interaction probabilities $p = \rho$ and $q = \xi p$. In both figures, the quantities $\lambdain_i$ are iid sampled from $\Din = \Uni(1-c,1+c)$ and the $\lambdaout_i$ are all equal to one. In Figure~\ref{fig:optimal_error_rate_uniform_ones_varying_c}, we let $\xi = 1$ an vary $c$, while in Figure~\ref{fig:optimal_error_rate_uniform_ones_varying_xi} we let $c = 0.8$ and we vary $\xi$. The optimal error rates are computed using the formula obtained in Proposition~\ref{prop:recovery_rate_pabm}.
 }
 \label{fig:optimal_error_rate_uniform_ones}
\end{figure}

To show that these phenomena are not artifact of setting the $\lambdaout$ all equal to 1 and sampling the $\lambdain$ from a particular distribution, we also provide in Figure~\ref{fig:optimal_error_rate_other_distributions} plot of the optimal error rate (as given by the formula derived in Proposition~\ref{prop:recovery_rate_pabm}) when the coefficients $\lambdain$ and $\lambdaout$ are sampled from different distributions. 

 \begin{figure}[!ht]
 \centering
 \begin{subfigure}[b]{0.4\textwidth}
  \includegraphics[width=1.0\linewidth]{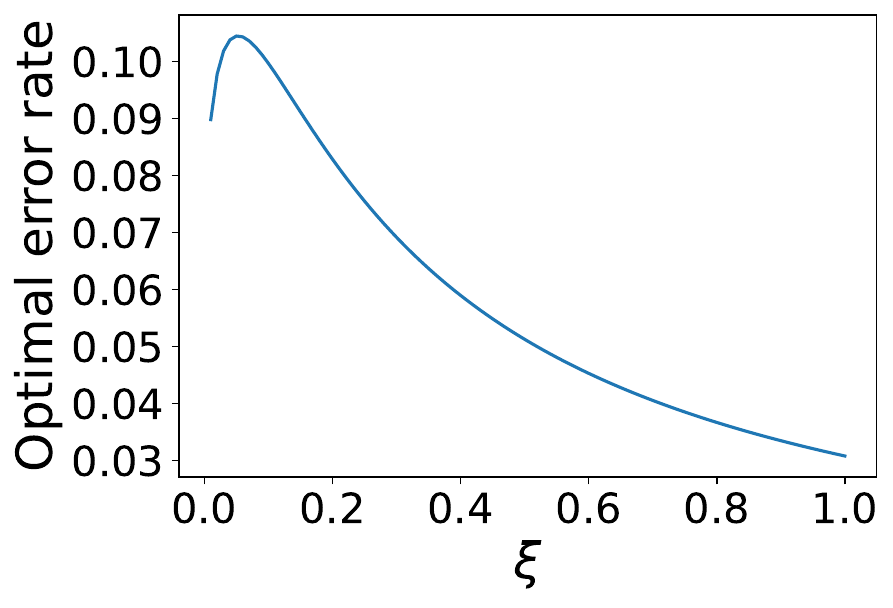}
  \caption{exponential}
  \label{fig:optimal_error_rate_other_distributions_exponential}
 \end{subfigure}
 \hfill 
 \begin{subfigure}[b]{0.4\textwidth}
  \includegraphics[width=1.0\linewidth]{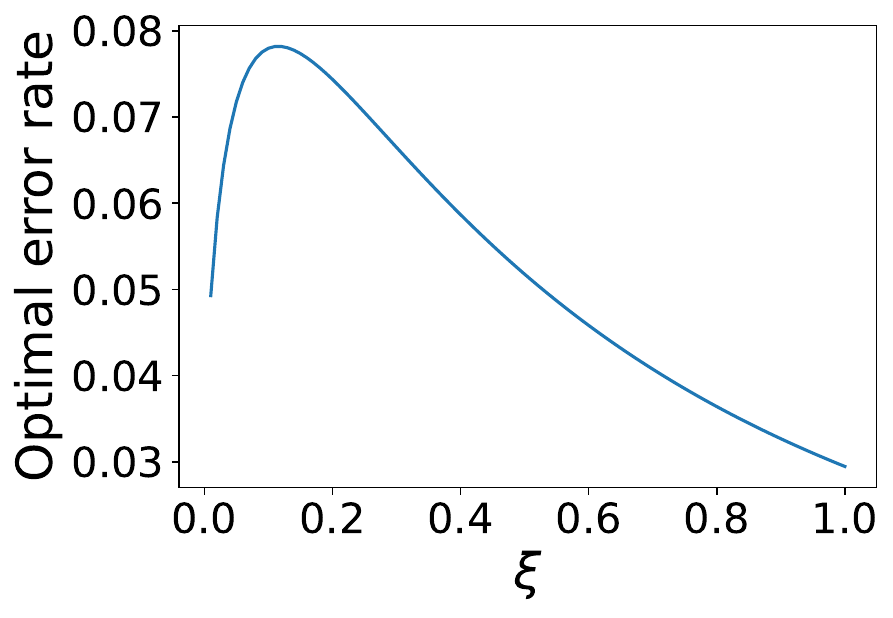}
  \caption{lognormal}
  \label{fig:optimal_error_rate_other_distributions_logNormal}
 \end{subfigure}
 \caption{Numerical values obtained for the optimal error rate $\frac{1}{n} \sum_i \exp(- \Chernoff(i,z^*))$ on PABM with homogeneous interactions. The matrix $P$ is given in Equation~\eqref{eq:experiments_homogeneous_PABM}, and we let $n=900$ vertices, $k=3$ clusters of same size, average edge density $\rho = 0.05$, and interaction probabilities $p = \rho$ and $q = \xi p$. In both figures, the quantities $\lambdain_i$ and $\lambdaout$ are iid sampled from a distribution $\cD$.
 Figure~\ref{fig:optimal_error_rate_other_distributions_exponential}: $\cD$ is the exponential distribution with mean 1. Figure~\ref{fig:optimal_error_rate_other_distributions_logNormal}: $\cD$ is the log-normal distribution with parameters $(\mu, \sigma) = (-1/2,1)$ (chosen so that the mean of the distribution is $1$). }
 \label{fig:optimal_error_rate_other_distributions}
\end{figure}

\section{Description of the Algorithms}
\label{appendix:description_algos}

\subsection{Variants of Spectral Clustering with \texorpdfstring{$k$}{k} Eigenvectors}

Algorithms~\ref{algo:sc_scikit-learn}, \ref{algo:sc_sbm}, and \ref{algo:sc_dcbm} provide the \textit{sklearn}, \textit{sbm}, and \textit{dcbm} variants of spectral clustering, respectively.

\begin{algorithm}[!ht]
\caption{Spectral Clustering: scikit-learn}
\label{algo:sc_scikit-learn}
\KwInput{ Adjacency matrix $A \in \R_+^{n \times n}$, number of clusters $k$ }
\KwOutput{ Predicted community memberships $\hz \in [k]^n$ }

Let $D = \diag( D 1_n )$ be the degree matrix

Compute the normalized Laplacian $\cL = I_n - D^{-1/2} A D^{-1/2}$ 

Compute the $k$ eigenvectors of $\cL$ associated to its $k$ smallest eigenvalues. Construct $V \in \R^{n \times k}$ using these eigenvectors as its columns. 
  
Let $\hz \in [k]^n$ be the output of Lloyd's algorithm (to solve the $k$-means problem) on the cloud of $k$-dimensional points $V_{i \cdot})_{i \in [n]}$. 
\end{algorithm}

\begin{algorithm}[!ht]
\caption{Spectral Clustering: standard block model variant}
\label{algo:sc_sbm}
\KwInput{ Adjacency matrix $A \in \R_+^{n \times n}$, number of clusters $k$ }
\KwOutput{ Predicted community memberships $\hz \in [k]^n$ }

Compute the $k$ eigenvectors $v_1, \cdots, v_k$ of $A$ associated to its $k$ largest eigenvalues (in absolute value) $|\sigma_1| \ge \cdots \ge |\sigma_k|$. Let $V = (v_1, \cdots, v_k) \in \R^{n \times k}$ and $\Sigma = \diag( \sigma_1, \cdots, \sigma_k) \in \R^{k \times k}$. 
  
Let $\hz \in [k]^n$ be the output of Lloyd's algorithm (to solve the $k$-means problem) on the cloud of $k$-dimensional points $( (V \Sigma)_{i \cdot})_{i \in [n]}$. 
\end{algorithm}

\begin{algorithm}[!ht]
\caption{Spectral Clustering: degree-corrected block model variant}
\label{algo:sc_dcbm}
\KwInput{Adjacency matrix $A \in \R_+^{n \times n}$, number of clusters $k$}
\KwOutput{ Predicted community memberships $\hz \in [k]^n$ }

Compute the $k$ eigenvectors $v_1, \cdots, v_k$ of $A$ associated to its $k$ largest eigenvalues (in absolute value) $|\sigma_1| \ge \cdots \ge |\sigma_k|$. Let $V = (v_1, \cdots, v_k) \in \R^{n \times k}$ and $\Sigma = \diag( \sigma_1, \cdots, \sigma_k) \in \R^{k \times k}$. 

Let $\hP = V \Sigma V^T$

Let $S_0 = \{ i \in [n] \colon \| P_{i\cdot} \|_1 = 0 \}$. Define $\tP_{i\cdot } = P_{i\cdot} / \| P_{i\cdot} \|_1 $ for $i \in S_0^c$ and $\tP_{i\cdot } = P_{i\cdot} $ for $i \in S_0$. 

Let $\hz \in [k]^n$ be the output of Lloyd's algorithm (to solve the $k$-means problem) on the cloud of $n$-dimensional points $( \hP_{i \cdot})_{i \in S_0^c}$ (note that we assign the vertices of $S_0$ arbitrarily). 
\end{algorithm}

\subsection{Variants of Spectral Clustering with \texorpdfstring{$k^2$}{k2} Eigenvectors}

In this section, we describe two algorithms proposed in the litterature for clustering PABM. 

\paragraph{Orthogonal Spectral Clustering} \cite{koo2023popularity} observed that PABM is a special case of the Generalized Random Dot Product Graph (GRDPG) for which the latent position vectors lie in distinct orthogonal subspaces, each subspace corresponding to a community. This leads to Algorithm~\ref{algo:osc}. 

\begin{algorithm}[!ht]
\caption{Orthogonal Spectral Clustering}
\label{algo:osc}
\KwInput{Adjacency matrix $A \in \R_+^{n \times n}$, number of clusters $k$ }
\KwOutput{Predicted clusters $\hz \in [k]^n$ }

Compute the eigenvectors of $A$ associated to its $k(k + 1)/2$ most positive eigenvalues and $k(k - 1)/2$ most negative eigenvalues. Construct $V \in \R^{n \times k^2}$ using these eigenvectors as its columns. 

Compute $B = |nV V^T | \in \R^{n \times n}$, applying $| \cdot |$ entry-wise. 
  
Let $\hz \in [k]^n$ be the output of spectral clustering (see Algorithm~\ref{algo:sc_scikit-learn}) applied on the graph whose adjacency matrix is $B$. 
\end{algorithm}



\paragraph{Subspace Spectral Clustering}

\cite{noroozi2021estimation} proposes another approach to cluster PABM. In particular, they notice that the expected adjacency matrix of a PABM has a rank between $k$ and~$k^2$ and is composed of subspaces. In particular, two vertices in the same community belong to the same subspace. This motivates the usage of subspace clustering, as opposed to $k$-means, for clustering the cloud of point obtained via the spectral embedding. For subspace clustering, we use the implementation provided in~\cite{you2016scalable} and available at \url{https://github.com/ChongYou/subspace-clustering}, and we refer to~\cite{elhamifar2013sparse} for an introduction on (sparse) subspace clustering. We summarized this in Algorithm~\ref{algo:sc_pabm}. 

\begin{algorithm}[!ht]
\caption{Subspace Clustering on Spectral Embedding}
\label{algo:sc_pabm}
\KwInput{Adjacency matrix $A \in \R_+^{n \times n}$, number of clusters $k$, embedding dimension $d$ (default: $d = k^2$) }
\KwOutput{Predicted clusters $\hz \in [k]^n$ }

Compute the $d$ eigenvectors $v_1, \cdots, v_d$ of $A$ associated to its $d$ largest eigenvalues (in absolute value) $|\sigma_1| \ge \cdots \ge |\sigma_d|$. Construct $V = (v_1, \cdots, v_d) \in \R^{n \times k}$ and $\Sigma = \diag( \sigma_1, \cdots, \sigma_d)$. 

Let $\hz \in [k]^n$ be the output of \textit{subspace clustering} on the cloud of $d$-dimensional points $( (V \Sigma)_{i \cdot})_{i \in [n]}$. 
\end{algorithm}

\subsection{Additional Clustering Algorithms}
\label{appendix:additional_clustering_algos}

The algorithm from~\cite{bhadra2025unified} is an iterative community detection method designed for the Popularity-Adjusted Block Model (PABM). It begins by computing an adjacency spectral embedding of the network into a low-dimensional space of dimension $d$ (where typically $d = k^2$). For each tentative community, a subspace is estimated via singular value decomposition of the node embeddings in that cluster. The algorithm then greedily reassigns nodes to the community whose subspace yields the smallest projection error, thereby minimizing the objective function. This process iterates until node assignments stabilize, yielding a community structure tailored to the PABM. Although the original paper does not assign a name to the algorithm, we refer to it as Greedy Subspace Projection Clustering (\textit{gspc}). Algorithm~\ref{alg:gspc} provides the pseudo-code.

\begin{algorithm}[!ht]
\caption{Greedy Subspace Projection Clustering (\textit{gspc})}
\label{alg:gspc}
\KwIn{Adjacency matrix $A \in \mathbb{R}^{n \times n}$, number of communities $K$, embedding dimension $d$ (default: $d=k^2$), initial cluster labels $z^{(0)} \in [k]^n$}
\KwOut{Final cluster labels $\hz \in [k]^n$}

Compute adjacency spectral embedding $X \in \mathbb{R}^{n \times d}$ from $A$\;
Initialize cluster labels $\hz \gets z^{(0)}$\;

\Repeat{no label changes or maximum iterations reached}{
    \For{$k \gets 1$ \KwTo $K$}{
        Extract $X_k \gets \{x_i : \ell_i = k\}$\;
        Compute leading $d$ left singular vectors $U_k$ of $X_k$\;
    }
    \For{$i \gets 1$ \KwTo $n$}{
        \For{$k \gets 1$ \KwTo $K$}{
            Compute projection loss $L_{ik} \gets \|x_i - U_k U_k^\top x_i\|^2$\;
        }
        Update $\hz_i \gets \arg\min_k L_{ik}$\;
    }
}
\Return $\hz$\;
\end{algorithm}

Thresholded Cosine Spectral Clustering (\textit{tcsc}), proposed in~\cite{yuan2025strongly}, begins by computing the top $k^2$ eigenvectors of the adjacency matrix to capture structural information. Cosine similarities between eigenvector rows are then calculated and thresholded to suppress noise. Finally, Lloyd’s algorithm is applied to the thresholded similarity representation to output the predicted cluster labels. 
Finally, \cite{yuan2025strongly} also proposes to refine the cluster labels obtained by \textit{tcsc}. This leads to Refined Thresholded Cosine Spectral Clustering (\textit{r-tcsc}), which improve upon the initial labels from \textit{tcsc} by re-estimating block connection probabilities and then reassigning vertices to clusters according to a profile likelihood criterion. This refinement step reduces misclassifications and yields more accurate community recovery. Pseudo-code for \textit{tcsc} is provided in Algorithm~\ref{algo:tcsc}, and the reader is refered to \cite[Theorem~2]{yuan2025strongly} for the refinement step. 

\begin{algorithm}[!ht]
\caption{Thresholded Cosine Spectral Clustering (\textit{tcsc})}
\label{algo:tcsc}
\KwIn{Adjacency matrix $A \in \mathbb{R}^{n \times n}$, number of communities $k$}
\KwOut{Predicted clusters $\hz \in [k]^n$}

\BlankLine
Compute the top-$K^2$ eigenvectors of $A$ and form $U \in \mathbb{R}^{n \times K^2}$. \\

For each pair of rows $U_i, U_j$, compute the cosine similarity
$
S_{ij} = \frac{\langle U_i, U_j \rangle}{\|U_i\| \, \|U_j\|}.
$

Apply thresholding: set $S_{ij} = 0$ if $S_{ij} < \tau$, where $\tau$ is a data-driven threshold. \\

Apply Lloyd’s algorithm (\(k\)-means) to the rows of $S$ to obtain the cluster labels $\hz$.
\end{algorithm}

\subsection{Rank Analysis in PABM}
\label{subsec:rank_pabm}

For simplicity, let us consider a PABM with $k=3$ blocks, and suppose that vertices are ordered such that the first $n_1$ vertices are in the first cluster, the next $n_2$ vertices are in the second cluster, and the last $n_3 = n - n_1 - n_2$ vertices are in the third cluster. For any vertex $i \in [n]$, we denote by $r_i$ its rank-indexing of its cluster (that is, $r_i = i$ if $i$ is in cluster 1, $r_i = i-n_1$ if $i$ is in cluster 2, and $r_i = i-n_1-n_2$ if $i$ is in cluster 3). Denote $\Lambda^{(a,b)}$ the matrix of size $n_a$-by-$1$ such that $\Lambda^{(a,b)}_{r_i} = \lambda_{i b}$. We also assume that $B_{ab} = p 1\{a=b\} + q 1\{a \ne b \}$ with $p \ne q$. Then, the matrix $P$ is given by 
 \begin{align*}
  P = 
  \begin{pmatrix}
    p \Lambda^{(1,1)} (\Lambda^{(1,1)} )^T & 
    q \Lambda^{(1,2)} (\Lambda^{(2,1)} )^T & 
    q \Lambda^{(1,3)} (\Lambda^{(3,1)} )^T \\
    q \Lambda^{(2,1)} (\Lambda^{(1,2)} )^T & 
    p \Lambda^{(2,2)} (\Lambda^{(2,2)} )^T & 
    q \Lambda^{(2,3)} (\Lambda^{(3,2)} )^T \\
    q \Lambda^{(3,1)} (\Lambda^{(1,3)} )^T & 
    q \Lambda^{(3,2)} (\Lambda^{(2,3)} )^T & 
    p \Lambda^{(3,3)} (\Lambda^{(3,3)} )^T
  \end{pmatrix}.
 \end{align*}
 Thus, the matrix $P$ is composed of $k^2 = 9$ blocks of rank one. Excluding trivial cases, the rank of $P$ can take any value between $k = 3$ and $k^2 = 9$. 
 For example, if all the vectors $\Lambda^{(a,b)}$ are all-1 vectors, then $P$ has rank $1$. But, if $\Lambda^{(1,1)}$ contains entries that are not all equal to 1, the rank of $P$ increases to~$4$. Similarly, if both $\Lambda^{(1,1)}$ and $\Lambda^{(1,2)}$ contain non-constant entries, the rank of $P$ becomes 5, and so~on.

\section{Additional Numerical Experiments}

\subsection{Performance of \textit{tcsc} and \textit{gspc} }

In this section, we compare the accuracy obtained by \textit{tcsc} and \textit{gspc} with the accuracy of \textit{pabm} and \textit{osc} (and of \textit{sklearn} as a baseline). We sample PABM with homogeneous interactions, and take the same parameters as in Section~\ref{subsec:expe_synthetic}.

 \begin{figure}[!ht]
 \centering
 \begin{subfigure}[b]{0.4\textwidth}
  \includegraphics[width=1.0\linewidth]{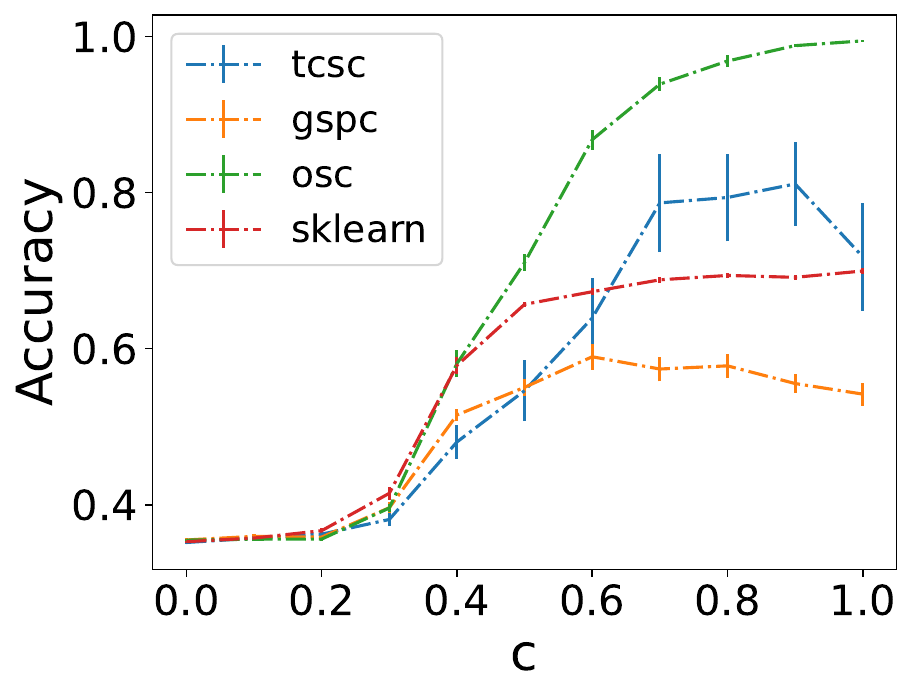}
  \caption{$\xi = 1$}
 \end{subfigure}
 \hfill 
 \begin{subfigure}[b]{0.4\textwidth}
  \includegraphics[width=1.0\linewidth]{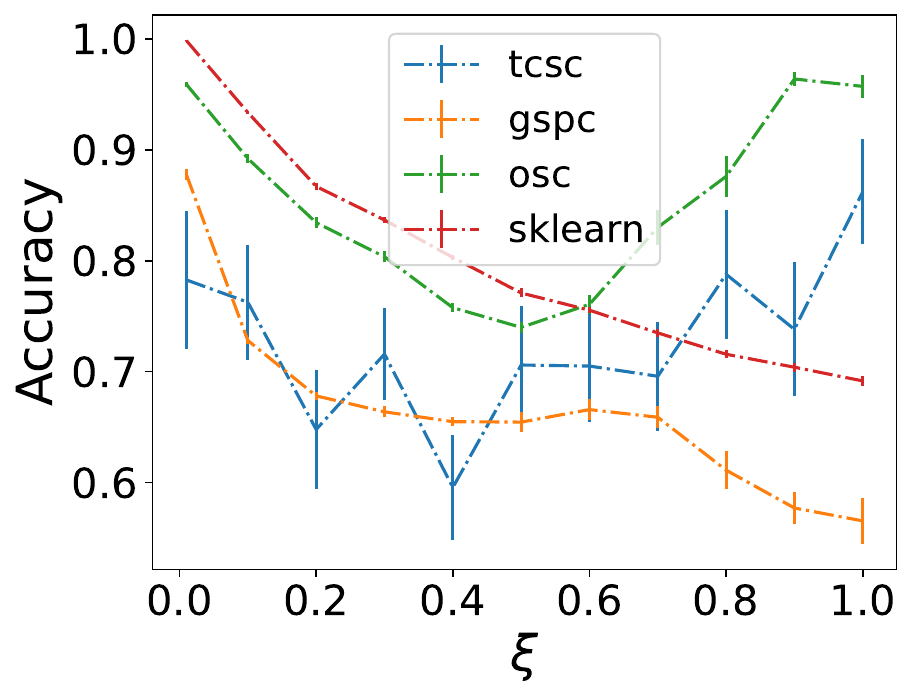}
  \caption{$c = 0.8$}
 \end{subfigure}
 \caption{Performance of graph clustering on homogeneous PABM, where the matrix~$P$ is given in Equation~\eqref{eq:experiments_homogeneous_PABM}. We sampled graphs with $n=900$ vertices in $k=3$ clusters of same size, average edge density $\rho = 0.05$. In both figures, the $\lambdain_i$ are iid sampled from $\Uni(1-c,1+c)$ and $\lambdaout_i = 1$ for all $i$. Accuracy is averaged over 15 realizations, and error bars show the standard errors.}
 \label{fig:homogeneous_uniform_pabm_algos}
\end{figure}

 \begin{figure}[!ht]
 \centering
 \begin{subfigure}[b]{0.4\textwidth}
  \includegraphics[width=1.0\linewidth]{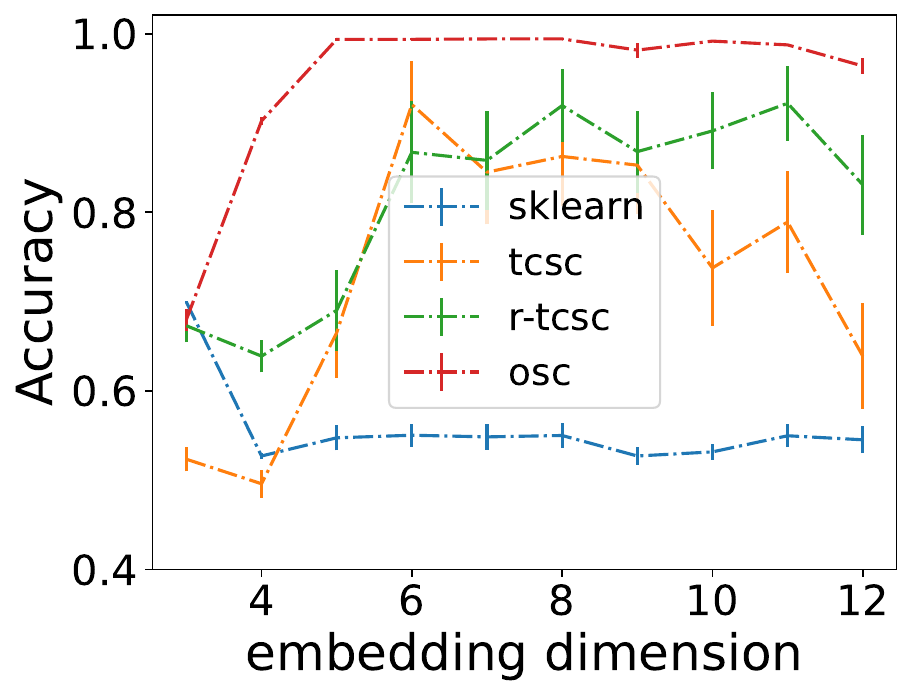}
  \caption{$\lambdaout_i = 1$}
 \end{subfigure}
 \hfill 
 \begin{subfigure}[b]{0.4\textwidth}
  \includegraphics[width=1.0\linewidth]{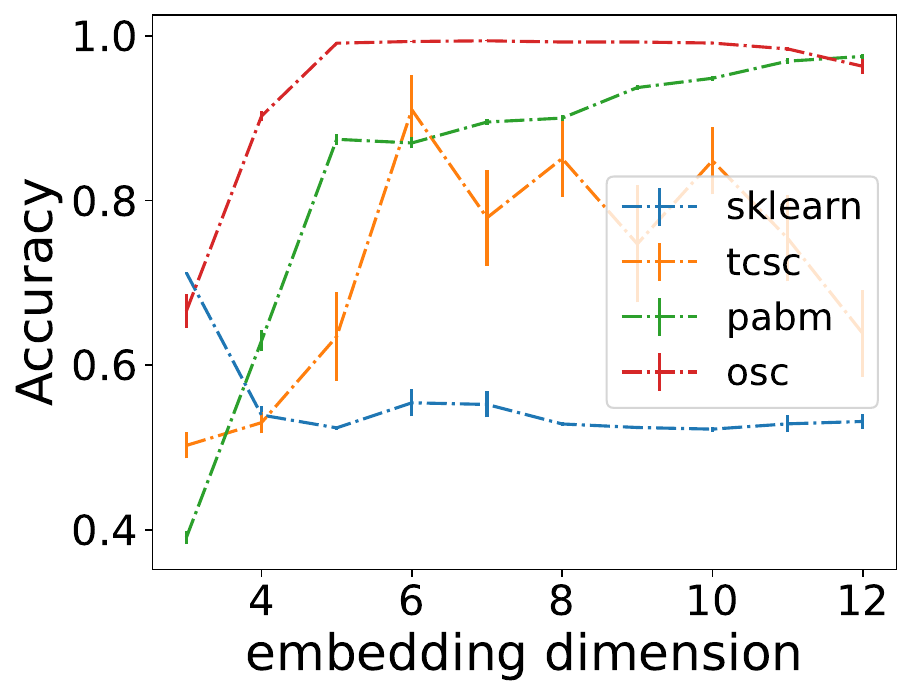}
  \caption{$\lambdaout_i \sim \Uni(0,2)$}
 \end{subfigure}
 \caption{Effect of the embedding dimension on the performance of graph clustering on homogeneous PABM, where the matrix~$P$ is given in Equation~\eqref{eq:experiments_homogeneous_PABM}. We sampled graphs with $n=900$ vertices in $k=3$ clusters of same size, average edge density $\rho = 0.05$. In both figures, the $\lambdain_i$ are iid sampled from $\Uni(0,2)$. Accuracy is averaged over 15 realizations, and error bars show the standard errors.}
\end{figure}

\subsection{Numerical Experiments on Heterogeneous PABM}

 We generate the coefficients $(\lambda_{ia})_{ i \in [n], a \in [k]}$ independently from each other and from a distribution with mean $1$ and bounded support so that $\sup_{i,a} \lambda_{ia} < 1 / \sqrt{ \rho }$, and let 
\begin{align}
\label{eq:experiments_heterogeneous_PABM}
 P_{ij} \weq 
 \begin{cases}
  \lambda_{iz_j^*} \lambda_{j z_i^*} \rho & \text{ if } z_i^* = z_j^*, \\ 
  \lambda_{iz_j^*} \lambda_{iz_i^*} \xi \rho & \text{ otherwise.}
 \end{cases}
\end{align}
To generate the $\lambda_{ia}$, we consider the following three distributions: Pareto with exponent $1.5$, log-normal with location $0$ and shape $1$ and exponential with parameter $1$. The support of these distributions is unbounded. To avoid having values too low and too large for the coefficients $\lambda_{ia}$, we sample a random variable $v_{ia}$ following one of these three distributions, and let 
\[
 \lambda_{ia} \weq
 \begin{cases}
   v_{ia} & \text{ if } v_{ia} \in [\tau_{\min}, \tau_{\max}], \\
   \tau_{\min} & \text{ if } v_{ia} < \tau_{\min}, \\
   \tau_{\max} & \text{ if } v_{ia} > \tau_{\max}.
 \end{cases} 
\]
In all experiments, we set $\tau_{\min} = 0.05$ and $\tau_{\max} = 5$. Finally, we normalize the $\lambda_{ia}$ to ensure that $\sum_i \lambda_{ia} = 1$ for all $a\in[k]$. 
Figure~\ref{fig:heterogeneous_xi} show that \textit{pabm} and \textit{osc} almost always outperform the \textit{sbm} and \textit{dcbm} variants. 

\begin{figure}[!ht]
 \centering
 \begin{subfigure}[b]{0.32\textwidth}
  \includegraphics[width=1.0\linewidth]{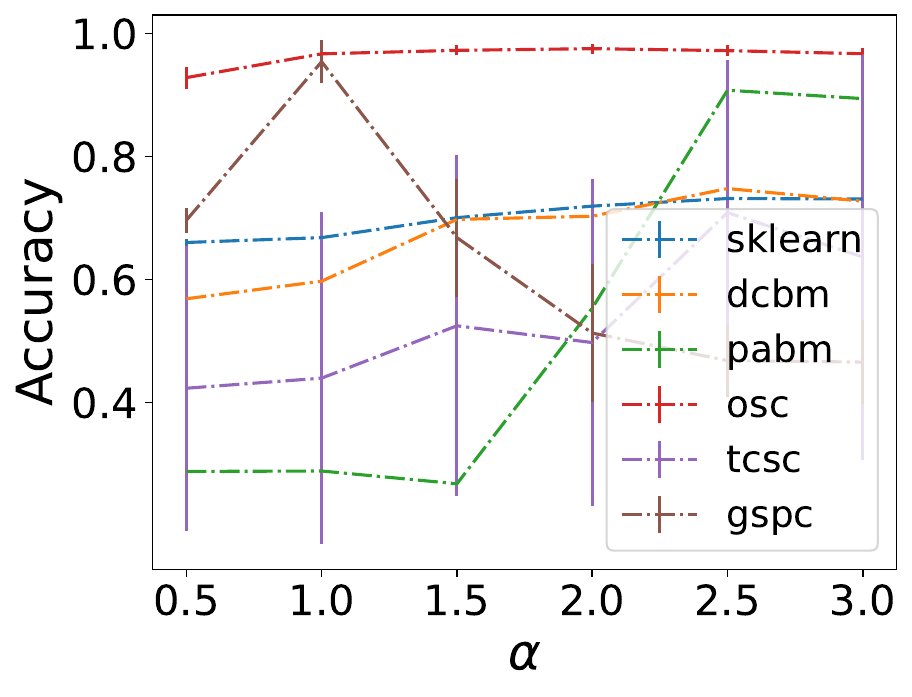}
  \caption{Pareto}
 \end{subfigure}
 \hfill 
 \begin{subfigure}[b]{0.32\textwidth}
  \includegraphics[width=1.0\linewidth]{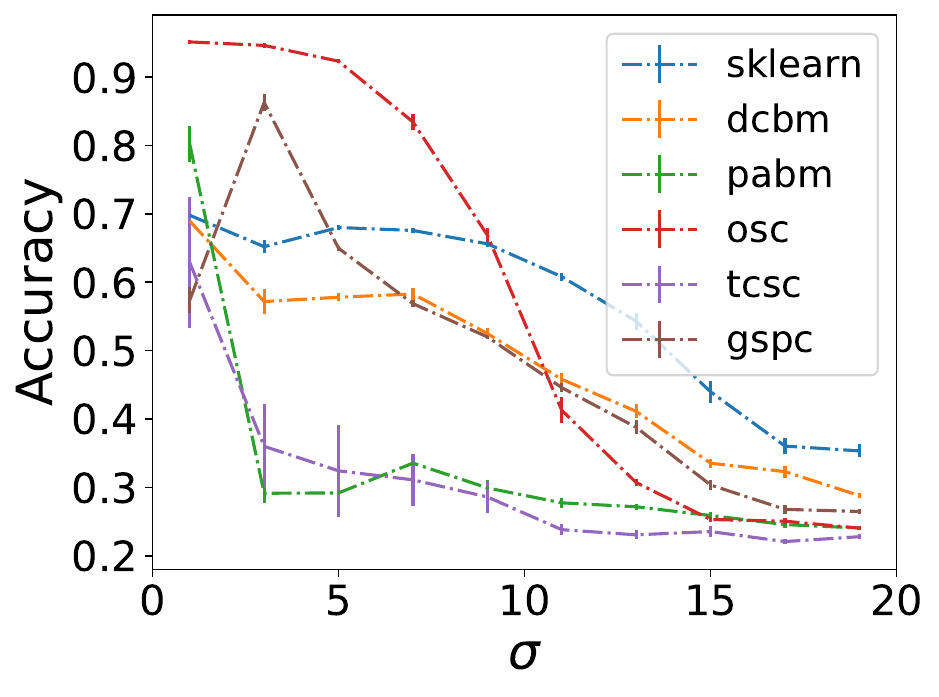}
  \caption{log-normal}
 \end{subfigure}
 \hfill 
 \begin{subfigure}[b]{0.32\textwidth}
  \includegraphics[width=1.0\linewidth]{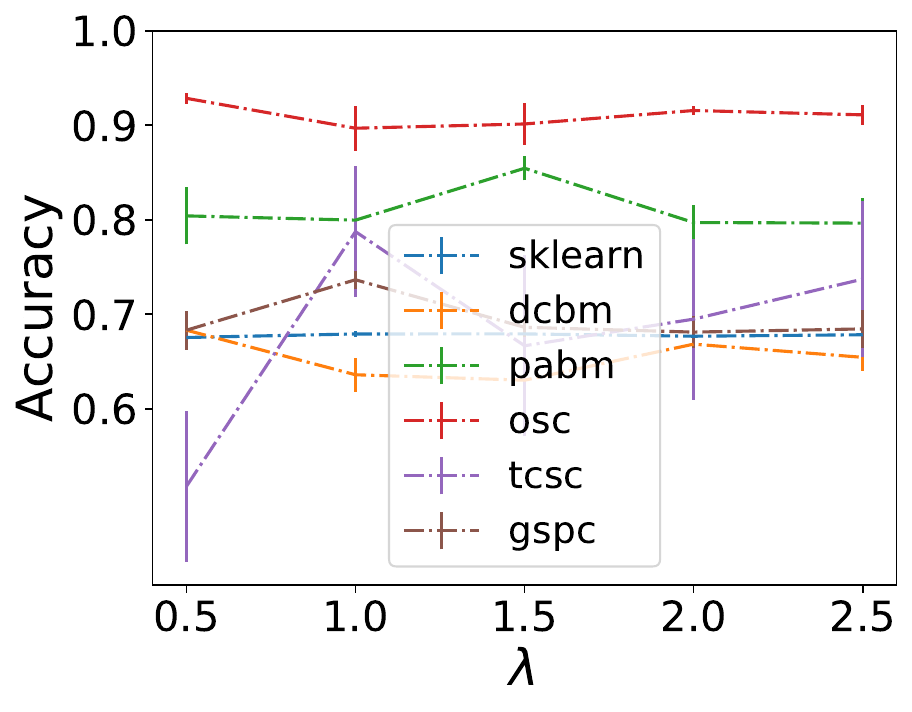}
  \caption{exponential}
 \end{subfigure}
 \caption{Performance of clustering algorithms on heterogeneous PABM, where the matrix $P$ is given in~\eqref{eq:experiments_heterogeneous_PABM} with $\rho = 0.05$, and the $\lambda_{ia}$ coefficients are sampled as described in the text. The curve show the average accuracy on 10 realization of PABM with $n=2000$ vertices in $k=5$ clusters of same size. Error bars show the standard errors (over 15 realizations).}
 \label{fig:heterogeneous_xi}
\end{figure}

\subsection{Real Data Sets Description}

Table~\ref{tab:statistics_real_datasets} provides some statistics about the graph used. For all graphs, we only considered the largest connected components. Moreover, for \textit{LiveJournal} data set, we extract the two largest clusters. Finally, for MNIST, FashionMNIST and Cifar10, we first embed the images into a low-dimensional space and we consider the $k$-nearest neighbor graph (with $k=10$) obtained from $n=10,000$ images. We use the embedding provided in the \textit{graphlearning} package.\footnote{\url{https://pypi.org/project/graphlearning/}.}

\label{appendix:real_datasets_description}
\begin{table}[!ht]
 \centering
 \begin{tabular}{ c ccccc c } \toprule
  data set & $n$ & $|E|$ & $k$ & $\bd$ & $\sqrt{ \overline{d^2} - (\bd)^2 }$ & Reference \\ \midrule
  political blog & 1,222 & 16,714 & 2 & 27.3 & 38.4 & \cite{adamic2005political} \\
  LiveJournal-top2 & 2,766 & 24,138 & 2 & 17.5 & 31.8 & \cite{backstrom2006group} \\
  citeseer & 2,110 & 3,668 & 6 & 3.5 & 4.0 & \cite{Getoor2005} \\
  cora & 2,485 & 5,069 & 7 & 4.1 & 5.4 & \cite{Getoor2005} \\
  MNIST & 10,000 & 85,938 & 10 & 17.2 & 5.0 & \cite{mnist} \\
  FashionMNIST & 10,000 & 83,486 & 10 & 16.7 & 4.0 & \cite{xiao2017fashion} \\
  CIFAR-10 & 10,000 & 97,044 & 10 & 19.4 & 8.8 & \cite{cifar} \\ \bottomrule
 \end{tabular}
 \caption{Summary of some statistics of the real data sets considered. The quantities $n$, $|E|$, and $k$ refer to the number of vertices $n$, of edges, and of clusters. The quantities $\bd$ and $\sqrt{ \overline{d^2} - (\bd)^2 }$ refer to the average and standard deviation of the degrees, respectively.}
 \label{tab:statistics_real_datasets}
\end{table}


\newpage
\section*{NeurIPS Paper Checklist}

\begin{enumerate}

\item {\bf Claims}
    \item[] Question: Do the main claims made in the abstract and introduction accurately reflect the paper's contributions and scope?
    \item[] Answer: \answerYes{} 
    \item[] Justification: The abstract and the introduction clearly state all the results of the paper.
    \item[] Guidelines:
    \begin{itemize}
        \item The answer NA means that the abstract and introduction do not include the claims made in the paper.
        \item The abstract and/or introduction should clearly state the claims made, including the contributions made in the paper and important assumptions and limitations. A No or NA answer to this question will not be perceived well by the reviewers. 
        \item The claims made should match theoretical and experimental results, and reflect how much the results can be expected to generalize to other settings. 
        \item It is fine to include aspirational goals as motivation as long as it is clear that these goals are not attained by the paper. 
    \end{itemize}

\item {\bf Limitations}
    \item[] Question: Does the paper discuss the limitations of the work performed by the authors?
    \item[] Answer: \answerYes{} 
    \item[] Justification: We discuss some limitations in the related work section as well as in the conclusion. We also mention in the numerical section than \textit{pabm} and \textit{osc} tend to be more computationally intensive than the other spectral clustering variants. 
    \item[] Guidelines:
    \begin{itemize}
        \item The answer NA means that the paper has no limitation while the answer No means that the paper has limitations, but those are not discussed in the paper. 
        \item The authors are encouraged to create a separate "Limitations" section in their paper.
        \item The paper should point out any strong assumptions and how robust the results are to violations of these assumptions (e.g., independence assumptions, noiseless settings, model well-specification, asymptotic approximations only holding locally). The authors should reflect on how these assumptions might be violated in practice and what the implications would be.
        \item The authors should reflect on the scope of the claims made, e.g., if the approach was only tested on a few datasets or with a few runs. In general, empirical results often depend on implicit assumptions, which should be articulated.
        \item The authors should reflect on the factors that influence the performance of the approach. For example, a facial recognition algorithm may perform poorly when image resolution is low or images are taken in low lighting. Or a speech-to-text system might not be used reliably to provide closed captions for online lectures because it fails to handle technical jargon.
        \item The authors should discuss the computational efficiency of the proposed algorithms and how they scale with dataset size.
        \item If applicable, the authors should discuss possible limitations of their approach to address problems of privacy and fairness.
        \item While the authors might fear that complete honesty about limitations might be used by reviewers as grounds for rejection, a worse outcome might be that reviewers discover limitations that aren't acknowledged in the paper. The authors should use their best judgment and recognize that individual actions in favor of transparency play an important role in developing norms that preserve the integrity of the community. Reviewers will be specifically instructed to not penalize honesty concerning limitations.
    \end{itemize}

\item {\bf Theory assumptions and proofs}
    \item[] Question: For each theoretical result, does the paper provide the full set of assumptions and a complete (and correct) proof?
    \item[] Answer: \answerYes{} 
    \item[] Justification: All theorems are carefully stated and the assumptions are also explained and discussed. 
    \item[] Guidelines:
    \begin{itemize}
        \item The answer NA means that the paper does not include theoretical results. 
        \item All the theorems, formulas, and proofs in the paper should be numbered and cross-referenced.
        \item All assumptions should be clearly stated or referenced in the statement of any theorems.
        \item The proofs can either appear in the main paper or the supplemental material, but if they appear in the supplemental material, the authors are encouraged to provide a short proof sketch to provide intuition. 
        \item Inversely, any informal proof provided in the core of the paper should be complemented by formal proofs provided in appendix or supplemental material.
        \item Theorems and Lemmas that the proof relies upon should be properly referenced. 
    \end{itemize}

    \item {\bf Experimental result reproducibility}
    \item[] Question: Does the paper fully disclose all the information needed to reproduce the main experimental results of the paper to the extent that it affects the main claims and/or conclusions of the paper (regardless of whether the code and data are provided or not)?
    \item[] Answer: \answerYes{} 
    \item[] Justification: All information to reproduce the experimental results is available in the paper (some details are in the Appendix). 
    \item[] Guidelines:
    \begin{itemize}
        \item The answer NA means that the paper does not include experiments.
        \item If the paper includes experiments, a No answer to this question will not be perceived well by the reviewers: Making the paper reproducible is important, regardless of whether the code and data are provided or not.
        \item If the contribution is a dataset and/or model, the authors should describe the steps taken to make their results reproducible or verifiable. 
        \item Depending on the contribution, reproducibility can be accomplished in various ways. For example, if the contribution is a novel architecture, describing the architecture fully might suffice, or if the contribution is a specific model and empirical evaluation, it may be necessary to either make it possible for others to replicate the model with the same dataset, or provide access to the model. In general. releasing code and data is often one good way to accomplish this, but reproducibility can also be provided via detailed instructions for how to replicate the results, access to a hosted model (e.g., in the case of a large language model), releasing of a model checkpoint, or other means that are appropriate to the research performed.
        \item While NeurIPS does not require releasing code, the conference does require all submissions to provide some reasonable avenue for reproducibility, which may depend on the nature of the contribution. For example
        \begin{enumerate}
            \item If the contribution is primarily a new algorithm, the paper should make it clear how to reproduce that algorithm.
            \item If the contribution is primarily a new model architecture, the paper should describe the architecture clearly and fully.
            \item If the contribution is a new model (e.g., a large language model), then there should either be a way to access this model for reproducing the results or a way to reproduce the model (e.g., with an open-source dataset or instructions for how to construct the dataset).
            \item We recognize that reproducibility may be tricky in some cases, in which case authors are welcome to describe the particular way they provide for reproducibility. In the case of closed-source models, it may be that access to the model is limited in some way (e.g., to registered users), but it should be possible for other researchers to have some path to reproducing or verifying the results.
        \end{enumerate}
    \end{itemize}

\item {\bf Open access to data and code}
    \item[] Question: Does the paper provide open access to the data and code, with sufficient instructions to faithfully reproduce the main experimental results, as described in supplemental material?
    \item[] Answer: \answerYes{} 
    \item[] Justification: The code to reproduce the experiments is available. Furthermore, all datasets considered are fairly standard (and they are directly available in the code we provide). 
    \item[] Guidelines:
    \begin{itemize}
        \item The answer NA means that paper does not include experiments requiring code.
        \item Please see the NeurIPS code and data submission guidelines (\url{https://nips.cc/public/guides/CodeSubmissionPolicy}) for more details.
        \item While we encourage the release of code and data, we understand that this might not be possible, so “No” is an acceptable answer. Papers cannot be rejected simply for not including code, unless this is central to the contribution (e.g., for a new open-source benchmark).
        \item The instructions should contain the exact command and environment needed to run to reproduce the results. See the NeurIPS code and data submission guidelines (\url{https://nips.cc/public/guides/CodeSubmissionPolicy}) for more details.
        \item The authors should provide instructions on data access and preparation, including how to access the raw data, preprocessed data, intermediate data, and generated data, etc.
        \item The authors should provide scripts to reproduce all experimental results for the new proposed method and baselines. If only a subset of experiments are reproducible, they should state which ones are omitted from the script and why.
        \item At submission time, to preserve anonymity, the authors should release anonymized versions (if applicable).
        \item Providing as much information as possible in supplemental material (appended to the paper) is recommended, but including URLs to data and code is permitted.
    \end{itemize}

\item {\bf Experimental setting/details}
    \item[] Question: Does the paper specify all the training and test details (e.g., data splits, hyperparameters, how they were chosen, type of optimizer, etc.) necessary to understand the results?
    \item[] Answer: \answerYes{} 
    \item[] Justification: Yes, all details regarding the experiments are specified in the paper. 
    \item[] Guidelines:
    \begin{itemize}
        \item The answer NA means that the paper does not include experiments.
        \item The experimental setting should be presented in the core of the paper to a level of detail that is necessary to appreciate the results and make sense of them.
        \item The full details can be provided either with the code, in appendix, or as supplemental material.
    \end{itemize}

\item {\bf Experiment statistical significance}
    \item[] Question: Does the paper report error bars suitably and correctly defined or other appropriate information about the statistical significance of the experiments?
    \item[] Answer: \answerYes{} 
    \item[] Justification: All numerical results are presented with error bars indicating the standard error of the mean. 
    \item[] Guidelines:
    \begin{itemize}
        \item The answer NA means that the paper does not include experiments.
        \item The authors should answer "Yes" if the results are accompanied by error bars, confidence intervals, or statistical significance tests, at least for the experiments that support the main claims of the paper.
        \item The factors of variability that the error bars are capturing should be clearly stated (for example, train/test split, initialization, random drawing of some parameter, or overall run with given experimental conditions).
        \item The method for calculating the error bars should be explained (closed form formula, call to a library function, bootstrap, etc.)
        \item The assumptions made should be given (e.g., Normally distributed errors).
        \item It should be clear whether the error bar is the standard deviation or the standard error of the mean.
        \item It is OK to report 1-sigma error bars, but one should state it. The authors should preferably report a 2-sigma error bar than state that they have a 96\% CI, if the hypothesis of Normality of errors is not verified.
        \item For asymmetric distributions, the authors should be careful not to show in tables or figures symmetric error bars that would yield results that are out of range (e.g. negative error rates).
        \item If error bars are reported in tables or plots, The authors should explain in the text how they were calculated and reference the corresponding figures or tables in the text.
    \end{itemize}

\item {\bf Experiments compute resources}
    \item[] Question: For each experiment, does the paper provide sufficient information on the computer resources (type of compute workers, memory, time of execution) needed to reproduce the experiments?
    \item[] Answer: \answerYes{} 
    \item[] Justification: We only used a laptop (CPU, no GPU) to perform the experiments. Some information on the time of execution are provided in the Appendix. 
    \item[] Guidelines:
    \begin{itemize}
        \item The answer NA means that the paper does not include experiments.
        \item The paper should indicate the type of compute workers CPU or GPU, internal cluster, or cloud provider, including relevant memory and storage.
        \item The paper should provide the amount of compute required for each of the individual experimental runs as well as estimate the total compute. 
        \item The paper should disclose whether the full research project required more compute than the experiments reported in the paper (e.g., preliminary or failed experiments that didn't make it into the paper). 
    \end{itemize}
    
\item {\bf Code of ethics}
    \item[] Question: Does the research conducted in the paper conform, in every respect, with the NeurIPS Code of Ethics \url{https://neurips.cc/public/EthicsGuidelines}?
    \item[] Answer: \answerYes{} 
    \item[] Justification: Our work fully conforms with the NeurIPS Code of Ethics.
    \item[] Guidelines:
    \begin{itemize}
        \item The answer NA means that the authors have not reviewed the NeurIPS Code of Ethics.
        \item If the authors answer No, they should explain the special circumstances that require a deviation from the Code of Ethics.
        \item The authors should make sure to preserve anonymity (e.g., if there is a special consideration due to laws or regulations in their jurisdiction).
    \end{itemize}

\item {\bf Broader impacts}
    \item[] Question: Does the paper discuss both potential positive societal impacts and negative societal impacts of the work performed?
    \item[] Answer: \answerNA{} 
    \item[] Justification: We do not directly discuss these impacts in the paper, as our main contribution is mostly a theoretic one. However, the impacts are the same as any (theoretic or applied) work on unsupervised learning. Indeed, graph clustering enhances the understanding of complex network structures in areas such as social sciences, biology, and information systems, potentially aiding in areas like public health interventions or knowledge discovery. However, we also acknowledge potential negative impacts, including privacy concerns and the risk of misuse in surveillance or profiling, especially when applied to social or communication networks without appropriate safeguards. These considerations highlight the importance of ethical deployment and transparency when applying such techniques. 
    \item[] Guidelines:
    \begin{itemize}
        \item The answer NA means that there is no societal impact of the work performed.
        \item If the authors answer NA or No, they should explain why their work has no societal impact or why the paper does not address societal impact.
        \item Examples of negative societal impacts include potential malicious or unintended uses (e.g., disinformation, generating fake profiles, surveillance), fairness considerations (e.g., deployment of technologies that could make decisions that unfairly impact specific groups), privacy considerations, and security considerations.
        \item The conference expects that many papers will be foundational research and not tied to particular applications, let alone deployments. However, if there is a direct path to any negative applications, the authors should point it out. For example, it is legitimate to point out that an improvement in the quality of generative models could be used to generate deepfakes for disinformation. On the other hand, it is not needed to point out that a generic algorithm for optimizing neural networks could enable people to train models that generate Deepfakes faster.
        \item The authors should consider possible harms that could arise when the technology is being used as intended and functioning correctly, harms that could arise when the technology is being used as intended but gives incorrect results, and harms following from (intentional or unintentional) misuse of the technology.
        \item If there are negative societal impacts, the authors could also discuss possible mitigation strategies (e.g., gated release of models, providing defenses in addition to attacks, mechanisms for monitoring misuse, mechanisms to monitor how a system learns from feedback over time, improving the efficiency and accessibility of ML).
    \end{itemize}
    
\item {\bf Safeguards}
    \item[] Question: Does the paper describe safeguards that have been put in place for responsible release of data or models that have a high risk for misuse (e.g., pretrained language models, image generators, or scraped datasets)?
    \item[] Answer: \answerNA{} 
    \item[] Justification: All datasets used are already available in the literature. 
    \item[] Guidelines:
    \begin{itemize}
        \item The answer NA means that the paper poses no such risks.
        \item Released models that have a high risk for misuse or dual-use should be released with necessary safeguards to allow for controlled use of the model, for example by requiring that users adhere to usage guidelines or restrictions to access the model or implementing safety filters. 
        \item Datasets that have been scraped from the Internet could pose safety risks. The authors should describe how they avoided releasing unsafe images.
        \item We recognize that providing effective safeguards is challenging, and many papers do not require this, but we encourage authors to take this into account and make a best faith effort.
    \end{itemize}

\item {\bf Licenses for existing assets}
    \item[] Question: Are the creators or original owners of assets (e.g., code, data, models), used in the paper, properly credited and are the license and terms of use explicitly mentioned and properly respected?
    \item[] Answer: \answerYes{} 
    \item[] Justification: We cite the original paper that produced the code and dataset that we used. 
    \item[] Guidelines:
    \begin{itemize}
        \item The answer NA means that the paper does not use existing assets.
        \item The authors should cite the original paper that produced the code package or dataset.
        \item The authors should state which version of the asset is used and, if possible, include a URL.
        \item The name of the license (e.g., CC-BY 4.0) should be included for each asset.
        \item For scraped data from a particular source (e.g., website), the copyright and terms of service of that source should be provided.
        \item If assets are released, the license, copyright information, and terms of use in the package should be provided. For popular datasets, \url{paperswithcode.com/datasets} has curated licenses for some datasets. Their licensing guide can help determine the license of a dataset.
        \item For existing datasets that are re-packaged, both the original license and the license of the derived asset (if it has changed) should be provided.
        \item If this information is not available online, the authors are encouraged to reach out to the asset's creators.
    \end{itemize}

\item {\bf New assets}
    \item[] Question: Are new assets introduced in the paper well documented and is the documentation provided alongside the assets?
    \item[] Answer: \answerYes{} 
    \item[] Justification: Our code is well documented. 
    \item[] Guidelines:
    \begin{itemize}
        \item The answer NA means that the paper does not release new assets.
        \item Researchers should communicate the details of the dataset/code/model as part of their submissions via structured templates. This includes details about training, license, limitations, etc. 
        \item The paper should discuss whether and how consent was obtained from people whose asset is used.
        \item At submission time, remember to anonymize your assets (if applicable). You can either create an anonymized URL or include an anonymized zip file.
    \end{itemize}

\item {\bf Crowdsourcing and research with human subjects}
    \item[] Question: For crowdsourcing experiments and research with human subjects, does the paper include the full text of instructions given to participants and screenshots, if applicable, as well as details about compensation (if any)? 
    \item[] Answer: \answerNA{} 
    \item[] Justification: The paper does not involve crowdsourcing nor research with human subjects.
    \item[] Guidelines:
    \begin{itemize}
        \item The answer NA means that the paper does not involve crowdsourcing nor research with human subjects.
        \item Including this information in the supplemental material is fine, but if the main contribution of the paper involves human subjects, then as much detail as possible should be included in the main paper. 
        \item According to the NeurIPS Code of Ethics, workers involved in data collection, curation, or other labor should be paid at least the minimum wage in the country of the data collector. 
    \end{itemize}

\item {\bf Institutional review board (IRB) approvals or equivalent for research with human subjects}
    \item[] Question: Does the paper describe potential risks incurred by study participants, whether such risks were disclosed to the subjects, and whether Institutional Review Board (IRB) approvals (or an equivalent approval/review based on the requirements of your country or institution) were obtained?
    \item[] Answer: \answerNA{} 
    \item[] Justification: The paper does not involve crowdsourcing nor research with human subjects.
    \item[] Guidelines:
    \begin{itemize}
        \item The answer NA means that the paper does not involve crowdsourcing nor research with human subjects.
        \item Depending on the country in which research is conducted, IRB approval (or equivalent) may be required for any human subjects research. If you obtained IRB approval, you should clearly state this in the paper. 
        \item We recognize that the procedures for this may vary significantly between institutions and locations, and we expect authors to adhere to the NeurIPS Code of Ethics and the guidelines for their institution. 
        \item For initial submissions, do not include any information that would break anonymity (if applicable), such as the institution conducting the review.
    \end{itemize}

\item {\bf Declaration of LLM usage}
    \item[] Question: Does the paper describe the usage of LLMs if it is an important, original, or non-standard component of the core methods in this research? Note that if the LLM is used only for writing, editing, or formatting purposes and does not impact the core methodology, scientific rigorousness, or originality of the research, declaration is not required.
    \item[] Answer: \answerNA{} 
    \item[] Justification: The core method development in this research does not involve LLMs as any important, original, or non-standard components. 
    \item[] Guidelines:
    \begin{itemize}
        \item The answer NA means that the core method development in this research does not involve LLMs as any important, original, or non-standard components.
        \item Please refer to our LLM policy (\url{https://neurips.cc/Conferences/2025/LLM}) for what should or should not be described.
    \end{itemize}

\end{enumerate}

\end{document}